%% file: main.tex
\documentclass[nohyperref]{article}

\usepackage{microtype}
\usepackage[dvipsnames]{xcolor}

\usepackage{graphicx}
\usepackage{booktabs} 

\usepackage{hyperref}

\usepackage[accepted]{icml2022}

\usepackage{amsmath}
\usepackage{amssymb}
\usepackage{mathtools}
\usepackage{amsthm}

\usepackage{enumerate,enumitem,mathrsfs,eucal,verbatim, bbm, derivative}
\usepackage{epigraph}
\usepackage{caption}
\usepackage{subcaption}
\usepackage{wrapfig}
\usepackage{graphbox} 
\usepackage{comment}

\usepackage[capitalize,noabbrev]{cleveref}

\theoremstyle{plain}

\newtheorem{thm}{Theorem}[section]

\newtheorem{prop}[thm]{Proposition}
\newtheorem{lemm}[thm]{Lemma}

\theoremstyle{definition}
\newtheorem{defn}{Definition}[section]

\newcommand{\dela}{\textnormal{Del}}
\newcommand{\circu}{\textnormal{Circ}}

\icmltitlerunning{Active Nearest Neighbor Regression Through Delaunay Refinement}

\begin{document}

\twocolumn[
\icmltitle{Active Nearest Neighbor Regression Through Delaunay Refinement}



\icmlsetsymbol{equal}{*}

\begin{icmlauthorlist}
\icmlauthor{Alexander Kravberg}{equal,yyy}
\icmlauthor{Giovanni Luca Marchetti}{equal,yyy}
\icmlauthor{Vladislav Polianskii}{equal,yyy}
\icmlauthor{Anastasiia Varava}{zzz}
\icmlauthor{Florian T. Pokorny}{yyy}
\icmlauthor{Danica Kragic}{yyy}
\end{icmlauthorlist}

\icmlaffiliation{yyy}{School of Electrical Engineering and Computer Science, Royal Institute of Technology (KTH), Stockholm, Sweden.}
\icmlaffiliation{zzz}{No affiliation}

\icmlcorrespondingauthor{Giovanni Luca Marchetti}{glma@kth.se}

\icmlkeywords{Delaunay triangulation, active learning, nearest neighbor}

\vskip 0.3in
]



\printAffiliationsAndNotice{\icmlEqualContribution} 

\input{sections/abstract}
\input{sections/introduction.tex}
\input{sections/related_work.tex}

\input{sections/method.tex}

\input{sections/experiments.tex}
\input{sections/discussion.tex}

\newpage
\bibliography{biblio}
\bibliographystyle{icml2022}

\newpage
\appendix
\onecolumn
\input{sections/appendix}

\end{document}

%% file: sections/abstract.tex
\begin{abstract}
We introduce an algorithm for active function approximation based on nearest neighbor regression. Our Active Nearest Neighbor Regressor (ANNR) relies on the Voronoi-Delaunay framework from computational geometry to subdivide the space into cells with constant estimated function value and select novel query points in a way that takes the geometry of the function graph into account. We consider the recent state-of-the-art active function approximator called DEFER, which is based on incremental rectangular partitioning of the space, as the main baseline. The ANNR addresses a number of limitations that arise from the space subdivision strategy used in DEFER. We provide a computationally efficient implementation of our method, as well as theoretical halting guarantees. Empirical results show that ANNR outperforms the baseline for both closed-form functions and real-world examples, such as gravitational wave parameter inference and exploration of the latent space of a generative model. 



\end{abstract}

%% file: sections/introduction.tex
\begin{figure}[tbh!]
    \centering
    \begin{subfigure}[b]{.37\linewidth}
        \centering
        \includegraphics[width=\linewidth]{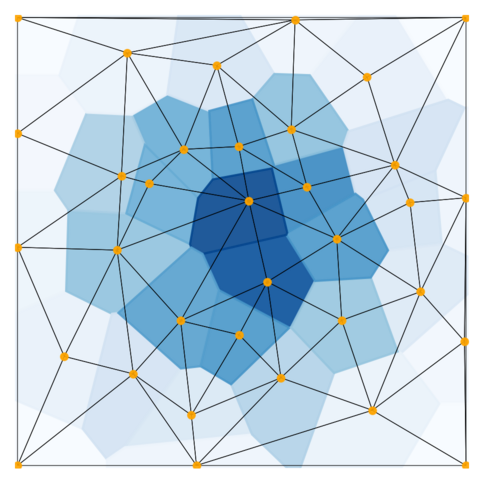}
    \end{subfigure}
    \begin{subfigure}[b]{.37\linewidth}
        \centering
        \includegraphics[width=\linewidth]{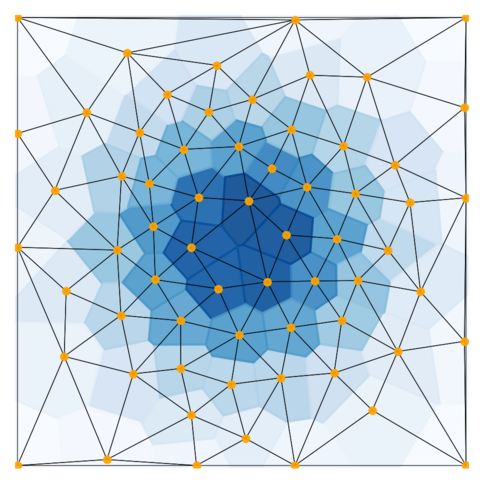}
    \end{subfigure}

    \begin{subfigure}[b]{.37\linewidth}
    \centering
    \includegraphics[width=\linewidth]{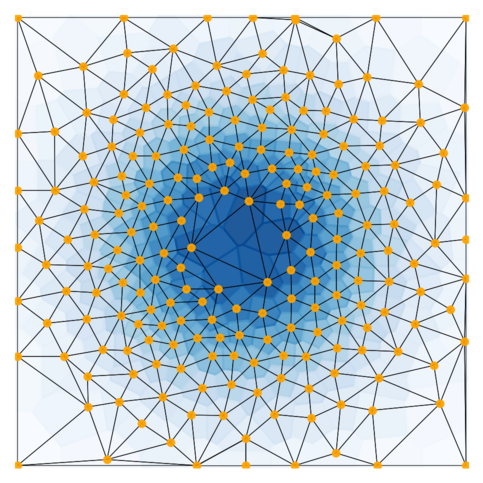}
    \end{subfigure}
       \begin{subfigure}[b]{.37\linewidth}
        \centering
        \includegraphics[width=\linewidth]{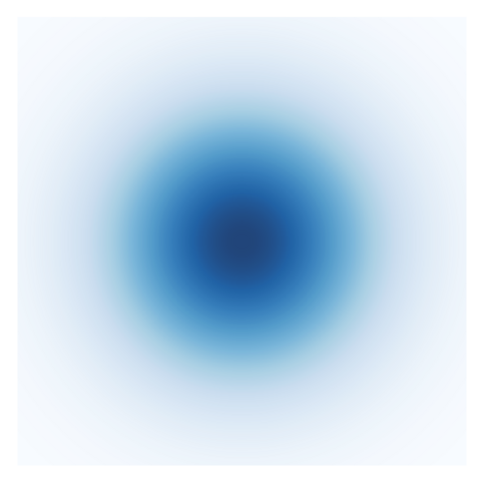}
    \end{subfigure}
    \caption{The ANNR progressively approximating a Gaussian function. The approximation is depicted in blue-scale, with the ground-truth function plotted in the lower-right corner. The Delaunay triangulation is depicted in black as well. }
    \label{fig:intro}
\end{figure}

\section{Introduction}
The need to \emph{actively} approximate a function by iteratively querying novel points from its domain  appears in a variety of theoretical and experimental areas, such as modern physics and astronomy  \cite{bilby, gw-bayesian}, chemistry  \cite{gubaev2018machine, ramakrishnan2015big} and density estimation \cite{active-learning, ABC, bayesian-active-learning-posterior}. Typically, the function to be approximated is not explicitly known, but can be evaluated at arbitrary points from its domain. Such formulation poses two intertwined problems: selecting points for evaluation (queries), and performing interpolation given a dataset with known function values. For the former, naive approaches such as sampling uniformly from the domain are computationally infeasible, especially for sparse and high-dimensional functions. Function evaluation is often expensive in real-life applications, as querying a single data point might require running a time- or resource-consuming experiment. It is thus crucial to design efficient strategies for selecting points to evaluate the function upon.


Recently, a scalable function approximator employing active querying has been proposed under the name DEFER \cite{defer}. It relies on a rectangular partitioning of the ambient space with the datapoints being the centers of the partitions, and approximates the (density) function piecewise constantly on each rectangle. DEFER outperforms state-of-the-art sampling methods in parameter inference tasks as well as in arbitrary function approximation. It is worth noting that despite  \citet{defer} generally address density estimation problems, DEFER is effectively an active function approximator and differs from traditional density estimators \cite{altman1992introduction} that are designed for static data.
 Using a rectangular partitioning, however, has a number of disadvantages. Rectangular approximations are not optimal for arbitrary shapes, especially in high dimensions. Indeed, shape approximation via rectangles becomes progressively worse as dimensions grow, which can be illustrated in a well-known 'spiky rectangle phenomenon' \cite{devi2016conceptualizing}.

We instead propose to upgrade the \emph{Nearest Neighbour Regressor} (NNR) \cite{altman1992introduction} to an active setting. The NNR is a function approximator which is locally constant on the Voronoi tessellation. The space is thus partitioned into Voronoi cells, which are arbitrary polytopes adaptive to the local geometry of data \cite{fortune1995voronoi}. Such a space partitioning and the corresponding locally-constant approximation address the aforementioned disadvantages of DEFER, which we empirically show in the present work. 


The core idea behind our active querying strategy is to look for points where the estimated function presents the largest variation. Such points are the most informative for the update of the approximator. This is done by considering the Delaunay triangulation \cite{fortune1995voronoi}, which is dual to the Voronoi tessellation. The triangulation allows to discretize the graph of the function and to look at its volume, which captures the function variation. Working with the graph of the function allows to balance between the exploration and exploitation strategies. This leads to a geometry-aware procedure deemed \emph{Active Nearest Neighbour Regressor} (ANNR). Voronoi tessellations and their dual Delaunay triangulations  enable to solve both the aforementioned problems of interpolation and querying within a single geometric framework. 

To make the computation of the Delaunay triangulation feasible in high dimensions, we build upon on the approximate stochastic method described in  \cite{polianskii2020voronoi}.  Additionally, we prove via geometric arguments that the volumes of the Delaunay simplices over the graph of $f$ get arbitrary small as the procedure progresses, obtaining halting guarantees for the ANNR. Our implementation of the ANNR is available at \url{https://github.com/vlpolyansky/annr}.

Our main contributions can be summarized as follows: \emph{(i)} a novel active querying procedure deemed ANNR exploiting the geometry of the graph of $f$ through the Delaunay triangulation; \emph{(ii)} an efficient high-dimensional implementation and a theoretical proof of halting for the ANNR; \emph{(iii)} an empirical investigation of the ANNR through a series of experiments demonstrating improved performance and robustness over the recently introduced active function approximator DEFER.


%% file: sections/related_work.tex
\section{Related Work}


{\bf Delaunay Triangulations.} The Delaunay triangulations were originally introduced in \citet{delaunay1934sphere} as natural triangulations of point-clouds in arbitrary dimension. Because of their remarkable geometrical properties, they have seen extensive applications within computer graphics \cite{bern1992mesh, chen2004mesh} and topological data analysis \cite{edelsbrunner2010computational}. The ANNR is in particular related to Delaunay-based techniques for mesh refinement \cite{shewchuk2002delaunay}, whose goal is to refine the Delaunay triangulation (mesh) of a coarse point-cloud by progressively adding novel points. Ruppert's algorithm \cite{ruppert1995delaunay} and Chew's second algorithm \cite{chew1993guaranteed} are the most popular to this end. The idea underlying both algorithms is to insert the circumcenters of poor quality Delaunay triangles. Although our method follows a similar pattern, we are concerned with the task of function approximation rather than mesh refinement. We thus consider the known values of the ground-truth function and query points in order to refine the approximation, rather than aiming to optimally fill the ambient space. Moreover, mesh refinement algorithms and applications to computer graphics in general are concerned with a two- or three-dimensional ambient space. In contrast, we provide a general framework and an efficient implementation of the ANNR in high dimensions.

{\color{olive}
}


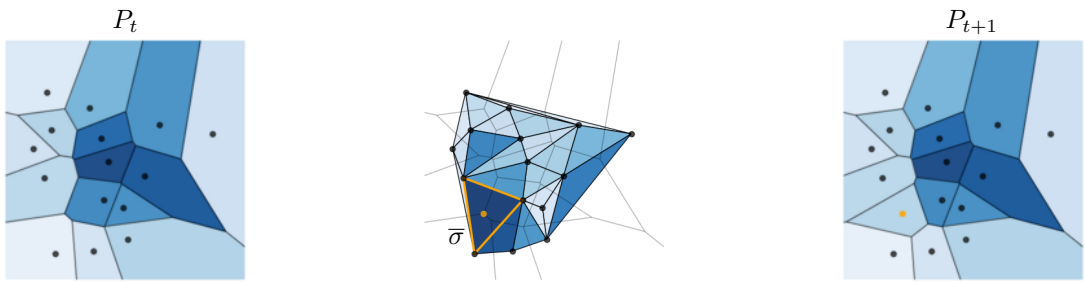
\begin{figure*}[tbh!]
    \centering
    \begin{subfigure}[b]{.32\linewidth}
        \centering
        \includegraphics[width=\linewidth]{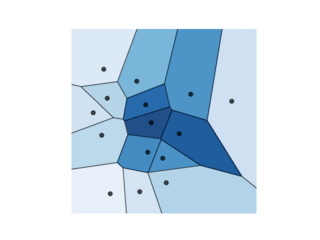}
    \end{subfigure}
    \begin{subfigure}[b]{.32\linewidth}
        \centering
        \includegraphics[width=\linewidth]{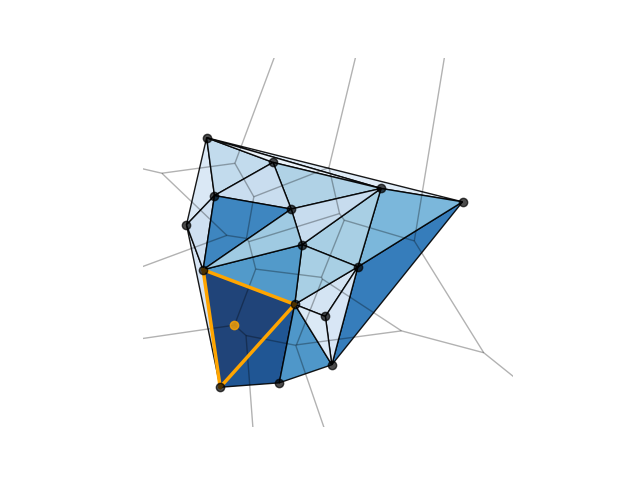}
    \end{subfigure}
        \begin{subfigure}[b]{.32\linewidth}
        \centering
        \includegraphics[width=\linewidth]{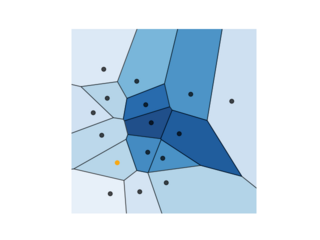}
    \end{subfigure}
        
      \begin{picture}(0,0)
    \put(-160,120){$P_t$}
      \put(155,120){$P_{t+1}$}
        \put(-33,38){$\overline{\sigma} $}
    \end{picture}
    \vspace{-2\baselineskip}
    
    \caption{A graphical depiction of the ANNR querying procedure. \textbf{Left}: the NNR (in blue-scale) of a two-dimensional dataset. \textbf{Center}: the corresponding Delaunay triangulation. The simplices are colored by the volume of their liftings, with the largest one highlighted in orange together with its circumcenter. \textbf{Right}: the updated NNR after the orange point has been added to the dataset. }
    \label{fig:annr}
\end{figure*}

{\bf Deep Active Learning.} Active querying strategies have been extensively studied in the context of deep learning. However, the vast majority of methods assume a predefined finite pool from which points can be queried ('pool-based active learning') or that datapoints can be sampled from the ground-truth distribution and then rejected for querying ('stream-based active learning') \cite{kontorovich2016, ren2021survey, settles2009active}. These methods typically deploy an 'acquisition function' representing some form of uncertainty that the desired query has to maximize \cite{lewis1994sequential, mackay1992information, cohn1994improving}. Picking an arbitrary point in the ambient space --sometimes referred to as 'membership query synthesis'-- is rarely considered in the literature and the corresponding subfield of active learning is relatively undeveloped. The reason is twofold: first, maximizing an acquisition function in the continuum for black-box models such as deep neural networks would require an additional expensive optimization procedure (gradient descent, for example) at each querying step. In contrast, simple but interpretable function approximators such as the piecewise linear ones considered in this work and in DEFER \cite{defer} allow to design querying strategies without the need of an acquisition function. Second, an arbitrary datapoint in the ambient space is likely to result in noise and thus to be uninformative to query, if even possible \cite{baum1992query}. Due to recent advances in generative modeling, this can be nowadays amended by looking for points to query in the latent space of a generative model \cite{zhu2017generative}.

{\bf Active Sampling in Bayesian Inference.} Acquisition functions based on uncertainty naturally occur in Bayesian inference, which is often used to approximate unknown distributions~\cite{active-learning}. Classical Bayesian methods have a number of disadvantages, such as assuming a certain form of a distribution, requiring computable gradients, and intractability. Recent developments in approximate Bayesian computation and Monte Carlo sampling methods address some of these issues~\cite{ABC-batch, VBMC, DNS}, albeit suffering from impractical computational complexity~\cite{gp-scalable}. When compared to DEFER, state-of-the-art Bayesian sampling methods showed comparable or worse performance, in particular, on multi-modal, sparse and discontinuous distributions~\cite{defer}. In addition, these methods can only be applied to a particular class of likelihood functions, and are typically designed to perform sampling from the estimated distributions rather than to estimate a (likelihood) function value at an arbitrary point.

%% file: sections/method.tex
\section{Method}\label{method}

Let $f : \  \mathcal{X} \rightarrow \mathbb{R}$ be a function defined on a connected metric space $(\mathcal{X}, d)$, where $d: \ \mathcal{X} \times \mathcal{X} \rightarrow \mathbb{R}_{\geq 0}$ denotes the distance. Given a finite set of points  $P \subseteq \mathcal{X}$ (referred to as datapoints) on which the values of $f$ are known, a natural way to extend $f$ beyond $P$  is to regress to the values at the nearest datapoint \cite{altman1992introduction}.

\begin{defn}
Let $x$ be a point such that all the distances $d(x,p)$, $p\in P$ are distinct. The value of the \emph{Nearest Neighbor Regressor} (NNR) $\widetilde{f}$ at $x$ coincides with the value of $f$ at the closest datapoint i.e.,
\begin{equation}\label{NNRformula}
    \widetilde{f}(x) = f(\overline{p}) , \ \ \overline{p}= \textnormal{argmin}_{p \in P}d(x,p).
\end{equation}
\end{defn}

Recall that the \emph{Voronoi cell} $C(p)$ of $p \in P$ contains the points in $\mathcal{X}$ that are closer to $p$ than to any other datapoint i.e., 
\begin{equation}
    C(p) = \{x\in \mathcal{X} \ | \ \forall q \in P \ \  d(x,q) \geq d(x,p) \}. 
\end{equation}

The Voronoi cells are closed, cover the ambient space $\mathcal{X}$ and intersect at their boundary. When $\mathcal{X} = \mathbb{R}^m$, they are arbitrary $m$-dimensional convex polytopes~\cite{fortune1995voronoi}. From the point of view of the Voronoi tessellation, the NNR is defined on the interior of Voronoi cells and is constant locally therein.

\subsection{Active Nearest Neighbor Regression}
\label{method-main}
In this work we upgrade the NNR to an active setting. 
A general active procedure consists in updating a dataset inductively by querying for the value of $f$ at a new datapoint based on the current dataset, on which $f$ is assumed to be known. Starting from an initial dataset $P_0$, this produces a sequence $\{ P_t \}_{t \in \mathbb{N}}$ obtained by adding a datapoint at each step: $P_{t+1} = P_t \cup \{ p_t \}$. We are now going to describe our proposed querying strategy.

To update the dataset, we first consider the triangulation which is dual to the Voronoi tessellation. From now on, we focus on the $m$-dimensional Euclidean space $\mathcal{X} = \mathbb{R}^m$. In the following we denote by $\langle \cdot \rangle$ the convex hull of a set. 

\begin{defn}
The \emph{Delaunay triangulation} $\dela_P$ generated by $P$ is the simplicial complex with vertices in $P$ that contains a $k$-dimensional simplex $\sigma = \langle v_0, \cdots, v_k \rangle $, $v_i \in P$, if and only if 
\begin{equation}
\bigcap_{0 \leq i \leq k} C(v_i) \not = \emptyset. 
\end{equation}
\end{defn}

If $P$ is in general position then the simplices in $\dela_P$ are non-degenerate. In that case, a remarkable property of the Delaunay triangulation is that no point in $P$ lies inside the circumsphere of an $m$-dimensional simplex $\sigma \in \dela_P$ \cite{fortune1995voronoi}. 

Our querying strategy intuitively relies on the variation of $f$ over the $m$-dimensional simplices of the Delaunay triangulation. Concretely, we compute the volume of the 'lifting' of such simplices to the graph $\Gamma_{\lambda f} = \{ (x, \lambda f(x)) \ | \ x\in \mathcal{X}\}$ of $ \lambda f$, where $\lambda \in \mathbb{R}_{\geq 0}$ is a hyperparameter. 

\begin{defn}
The \emph{lifting} via $\lambda f$ of an $m$-dimensional simplex $\sigma = \langle v_0, \cdots, v_m \rangle$ in $\mathbb{R}^m$ is the simplex in $\mathbb{R}^{m+1}$
\begin{equation}\label{lifting}
    \hat{\sigma} = \langle (v_0, \lambda f(v_0)), \cdots, (v_m, \lambda f(v_m)) \rangle. 
\end{equation}
\end{defn}

Our algorithm looks for the $m$-dimensional simplex $\sigma \in \dela_P$ that maximizes $\textnormal{Vol}(\hat{\sigma})$. To this end, the $k$-dimensional volume of a simplex $\sigma = \langle v_0, \cdots, v_k \rangle$ can be efficiently computed even for high dimensions via the \emph{Cayley-Menger determinant} \cite{sommerville1958introduction}: 
\begin{equation}\label{cayley}
\textnormal{Vol}(\sigma) = \sqrt{ \frac{(-1)^{k+1}}{2^k (k!)^2} \ \textnormal{det} M }. 
\end{equation}
Here, $M$ is the $(k+2) \times (k+2)$ matrix obtained by padding by a top row and a left column equal to $(0, 1, \cdots, 1)$ the matrix of mutual distances $d(v_i, v_j)^2$. \\

The role of the hyperparameter $\lambda$ is to 'sharpen' the lifted simplices. Since $\lambda$ controls the increment of $\lambda f $, $\textnormal{Vol}(\hat{\sigma})$ depends monotonically on $\lambda$. As $\textnormal{Vol}(\sigma)$ is constant, $\lambda \gg 0$ encourages simplices with high variation even if the base simplex $\sigma$ is small. On the other hand, $\textnormal{Vol}(\sigma) \sim \textnormal{Vol}(\hat{\sigma})$ for $\lambda \sim 0$, in which case the ANNR regularly explores the domain by looking for the largest simplices, disregarding $f$ and its variation completely. In summary,
$\lambda$ can be interpreted as governing a trade-off between exploitation of the estimated variation and domain exploration -- a classical compromise in active learning \cite{explore-exploit}. We make the latter statement formal in Sec.~\ref{geomintuition} and refer to Sec.~\ref{exp:lambda} for a further discussion around the practical choice of $\lambda$. 

Given the simplex $\overline{\sigma}$ that maximizes $\textnormal{Vol}(\hat{\sigma})$, the point we query is the dual of $\overline{\sigma}$ in the Voronoi tessellation. In other words, the novel query is the value of $f$ at the intersection between the $m+1$ Voronoi cells of the vertices of $\overline{{\sigma}}$. It is thus a point where the NNR (Eq.~\ref{NNRformula}) is maximally discontinuous, which motivates the need to gain information on $f$ around it. Geometrically, it coincides with the \emph{circumcenter} of $\overline{\sigma}$ (i.e., the point in $\mathbb{R}^m$ equidistant from the vertices of $\overline{\sigma}$) and we consequently denote it by $\circu(\overline{\sigma})$. Our querying strategy is graphical depicted in Fig.~\ref{fig:annr} and can be formally summarized as follows: 

\begin{equation}
p_{t + 1} = \circu(\overline{\sigma}), \ \  \overline{\sigma} = \textnormal{argmax}_{\sigma \in \dela_{P_t}} \textnormal{Vol}(\hat{\sigma}).
\end{equation}

 We stop the iteration when the maximum volume of a lifting reaches a given threshold $\varepsilon >0$. If care is taken in order to constrain the dataset in a compact region (which is discussed in Sec.~\ref{boundingsection}), the algorithm is guaranteed to halt with mild assumptions on $f$ (see Sec.~\ref{convergence}). The overall procedure is summarized in the pseudocode Alg.~\ref{alg:nnr}. 

\begin{algorithm}[tb]
   \caption{Active Nearest Neighbor Regression (ANNR)}
   \label{alg:nnr}
\begin{algorithmic}
   \STATE Initialize  $P$
   \STATE $maxVol=\varepsilon$
   \WHILE{$maxVol \geq \varepsilon$}
    \STATE Compute the Delaunay triangulation $\dela_{P}$
    \STATE   $maxVol = 0$
    \FOR{$\sigma \in \dela_{P}$ of dimension $m$}
    \STATE Compute the volume $\textnormal{Vol}(\hat{\sigma})$ via Eq.~\ref{cayley}
   \IF{$\textnormal{Vol}(\hat{\sigma}) > maxVol$}
    \STATE    $maxVol = \textnormal{Vol}(\hat{\sigma})$ 
    \STATE   $maxSimplex = \hat{\sigma}$
   \ENDIF    
    \ENDFOR
    \STATE Add the circumcenter of $maxSimplex$ to $P$
   \ENDWHILE
\end{algorithmic}
\end{algorithm}

\subsection{Bounding the Query}\label{boundingsection}

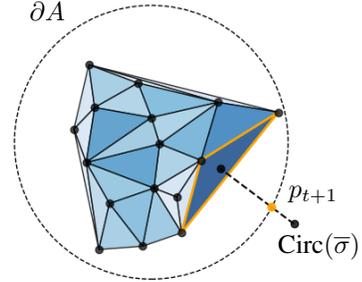
\begin{figure}[th!]
    \centering
    \includegraphics[width=.6\linewidth]{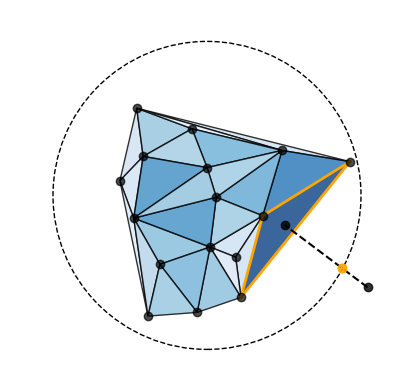}
          \begin{picture}(0,0)
    \put(-120,110){$\partial A$}
      \put(-22,43){$p_{t+1}$}
        \put(-26,22){$\circu(\overline{\sigma})$}
    \end{picture}
    
    \caption{Depiction of the query in the case the circumcenter falls outside of the bounding region $A$. }
    \label{fig:samples}
\end{figure}

The circumcenter of a simplex can possibly lie outside of the simplex itself. When a $m$-simplex is close to being degenerate (i.e., contained in a $(m- 1)$-dimensional affine subspace), the circumcenter tends to escape in an infinite direction. It is thus necessary to limit the expansion of the dataset within a compact region. 

We propose to fix a convex compact set $A \subseteq \mathbb{R}^m$ containing the initialization $P_0$ and to constrain $P_t$ to be contained in $A$ in the following way. Suppose that $\overline{\sigma}$ is the simplex in $\dela_{P_t}$ whose lifting has the largest volume and that $\circu(\overline{\sigma}) \not \in A$. Instead of setting $p_{t+1} =\circu(\overline{\sigma})$ we set $p_{t+1}$ as the intersection between the boundary $\partial A$ of $A$ and the segment connecting the barycenter of $\overline{\sigma}$ (which is contained in $\overline{\sigma}$) and $\circu(\overline{\sigma})$. Such segment is known as 'Euler line' of $\overline{\sigma}$ \cite{kimberling1998triangle} and the intersection between it and $\partial A$ is unique due to the convexity of $A$. The intuition behind this choice for $p_{t+1}$ is that it is the furthest point in $A$ from (the barycenter of) $\overline{\sigma}$ in the direction of the circumcenter of the latter. Moreover, the theoretical halting guarantees discussed in Sec.~\ref{convergence} crucially rely on the procedure described here. A graphical depiction is provided in Fig.~\ref{fig:samples}.

The bounding set $A$ additionally enables to initialize the dataset $P_0$ by uniformly sampling from $A$. Since the Delaunay triangulation covers the convex hull of its vertices, one can additionally enlarge $P_0$ so that $\langle P_0 \rangle =A$ and thus $\langle P_0 \rangle =A$ for every $t$. This is always possible if  $A$ is polytopal. When $A$ is a hypercube (which is the case in all of our experiments) this can be done by adding its $2^m$ vertices to $P_0$. We stick to this form of initialization when implementing Alg.~\ref{alg:nnr}.


\subsection{Computing the Delaunay Triangulation}\label{secapprox}

\begin{figure}[h!]
    \centering
    \includegraphics[width=.5\linewidth]{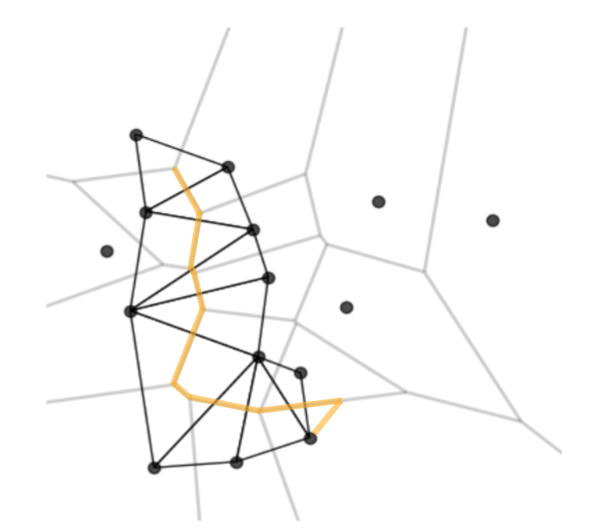}
    \caption{A depiction of a random walk along the boundary of the Voronoi cells (orange) together with the Delaunay simplices found along the way.}
    \label{fig:walk}
\end{figure}

A naive implementation of ANNR involving standard exact methods of constructing a Delaunay triangulation would present a significant computational challenge. Namely, the Delaunay triangulation is prohibitively expensive to compute in high dimensions as the number of simplices grows exponentially with respect to the dimension $m$ in a general scenario \cite{klee1979complexity}. \\

We propose to approximately infer $\dela_P$ via an adaptation of the stochastic ray-casting technique first considered in \citet{rushdi2017vps} and later expanded in \citet{polianskii2020voronoi}. The central idea lies in performing a random walk along the boundary of the Voronoi cells in order to look for their vertices, which in turn correspond by duality to the desired Delaunay simplices. More precisely, a Markov Chain (MC) is constructed in the following way. First, a starting datapoint $x_0$ is picked uniformly from $P$. Then a random direction $\theta_0 \in \mathbb{S}^{m-1}$ is chosen and the intersection $x_1$ between the ray $\{x_0 + t\theta_0 \}_{t \in \mathbb{R}_{\geq 0}}$ cast from $x_0$ and a $(m-1)$-dimensional face of the polytopal Voronoi cell $C(x_0)$ is found by an explicit analytic expression. The procedure is repeated by additionally constraining the subsequent ray to lie on the aforementioned face, obtaining a point $x_2$ on a $(m-2)$-dimensional face of $C(x_0)$, and so on. After $m$ rounds of iteration, a vertex $x_m$ of the Voronoi cell (i.e., a $0$-dimensional face) is obtained. By keeping track of the other Voronoi cells to which the encountered faces belong to, the (vertices of the) Delaunay simplex corresponding to $x_m$ are automatically available. After that, the random walk continues on the $1$-skeleton of the Voronoi tessellation in an analogous manner, possibly departing from $C(x_0)$ and finding other Delaunay simplices. The result is a subset of the Delaunay triangulation which approximates the whole $\dela_P$. By deploying a nearest-neighbor lookup structure such as a $k$-d tree~\cite{bentley1975multidimensional}, every ray intersection described above can be performed with a complexity that depends \textit{logarithmically} on the number of datapoints. We refer to \citet{polianskii2020voronoi} for further details. A graphical depiction of the described procedure is presented in Fig.~\ref{fig:walk}. \\

The vertices of the Voronoi cells corresponding to the found Delaunay simplices (i.e., the circumcenters of the latter) are obtained as a byproduct, thus there is no need for a further computation of the point to query. Because of the iterative nature of the search for Delaunay simplices, we integrate the latter with the main loop in Alg.~\ref{alg:nnr} for a further computational improvement. We further heuristically adjust the initialization of the random walks by picking datapoints close to barycenters of the simplices with the highest lifted volume at the previous step. As suggested in \cite{polianskii2020voronoi}, this can be accomplished via a 'visibility walk'~\cite{devillers2001walking} in the new $1$-skeleton towards such barycenters before initiating the random walk.

\subsection{Complexity Analysis}
In this section we provide a comparison of the computational complexity of the ANNR and DEFER for both data querying and evaluation of the approximated function.

Since each random walk has logarithmic complexity (Sec. \ref{secapprox}), querying a new datapoint has complexity $\mathcal{O} \left(  L\log |P| \right)$ for the ANNR, where $L$ is the number of MC steps performed. The latter is a hyperparameter typical for MC methods and its effect can be mitigated by running multiple walks in parallel. For DEFER, querying has complexity $\mathcal{O} \left( \log |P| \right)$. The  difference in complexity due to $L$ indeed reflects the price of approximating the Delaunay triangulation, as opposed to the exact geometry of rectangular partitions in DEFER. Additionally, the ANNR computes volumes of simplices, which has complexity $\mathcal{O}(m^3)$ w.r.t. the dimension $m$ due to the Cayley-Menger determinant (Eq. \ref{cayley}), in contrast to linear complexity for volumes of rectangles in DEFER. Overall, querying complexity of the ANNR is $\mathcal{O}(L(m^3  + \log |P|)) $.

Evaluating the approximated function has identical complexity $\mathcal{O}(\log |P| )$ with respect to the current dataset size $|P|$ in both methods. This is due to data structures such as $k$-d trees underlying both the nearest neighbor lookup for Voronoi cells in the ANNR and the rectangle lookup in DEFER. 




\section{Theoretical Results}
\subsection{Geometric Interpretation}\label{geomintuition}
In this section we present a geometric interpretation of the ANNR based on an inequality for volumes of graphs of functions. To this end, suppose that $\Omega \subseteq \mathbb{R}^m$ is an $m$-dimensional connected and compact submanifold (with boundary). For a smooth function $f : \ \Omega \rightarrow \mathbb{R}$ denote by $\Gamma_f $ its graph, which is an $m$-dimensional manifold. Additionally, denote by $f_\Omega = \frac{1}{\textnormal{Vol}(\Omega)}\int_\Omega f$ the average of $f$ over $\Omega$ and by $g_f$ the matrix $\nabla_f \otimes \nabla_f$ i.e.,  $(g_f)_{i,j} = \partial_{x_i}f \  \partial_{x_j}f$. 

\begin{prop}\label{interpretation}
There exists a constant $C > 0$ such that for every smooth function $f : \ \Omega \rightarrow \mathbb{R}$ the following inequality holds: 
\begin{equation}\label{inequal}
\log \textnormal{Vol}(\Gamma_f)  \geq C \|f - f_\Omega \|_2^2 + \log \textnormal{Vol}(\Omega) + o(\| g_f \|^2).
\end{equation}
\end{prop}

We refer to the Appendix for a proof. If $\Omega$ is a Voronoi cell then the NNR takes the value of $f$ at the only datapoint contained in the cell, which is a one-sample Monte Carlo estimate of $f_\Omega$. In other words, $f_\Omega \sim \widetilde{f}$ on $\Omega$. In light of Proposition~\ref{interpretation}, the (logarithm of the) volume of the graph of $f$ is approximately bounding the error between $f$ and its NNR estimation plus a term $\log \textnormal{Vol}(\Omega)$ penalizing domains with large volume. The latter can be thought as a form of regularization. The ANNR exploits this as it computes and approximation to the volume of the graph of (a multiple of) $f$. Such volume is not directly approximable when $\Omega$ is a Voronoi cell because of lack of data in $\Omega$. We thus compute it over the discretization of the manifold $\Gamma_f$ given by the lifted Delaunay triangulation $\{ \hat{\sigma} \}_{\sigma \in \dela_P}$. In other words, the volume is computed for the simplices formed by the datapoints in neighboring intersecting cells. 
    
Note that if we replace $f$ by $\lambda f$ for some $\lambda \in \mathbb{R}_{\geq 0}$ then the right hand-side of Eq.~\ref{inequal} becomes $C \lambda^2 \|f - f_\Omega \|_2^2 + \log \textnormal{Vol}(\Omega) + \lambda^4 o(\| g_f \|^2)$. A natural interpretation is that $\lambda$ controls the balance between error minimization (corresponding to the term $\|f - f_\Omega \|_2^2$) and exploration of the domain (corresponding to the regularization term $\log \textnormal{Vol}(\Omega)$). This motivates the insertion of the hyperparameter $\lambda$ in the ANNR (see Eq.~\ref{lifting}).

\subsection{Halting Guarantees}\label{convergence}
In this section we establish formal halting guarantees for the ANNR given a continuity assumption on the ground-truth function $f$. Since our algorithm stops when a fixed threshold $\varepsilon$ is met, halting can be equivalently reformulated as the vanishing of the corresponding lower limit. The halting guarantee is then given by the following. 

 
\begin{prop}\label{lowerlimit}
 Assume that the ground-truth function $f$ is Lipschitz and let $s_t = \textnormal{max}_{\sigma \in \dela_{P_t}} \textnormal{Vol}(\hat{\sigma})$. Then Alg.~\ref{alg:nnr} always halts for any $\varepsilon$ or, in other words,  $\underline{\textnormal{lim}}_{t \to \infty} s_t = 0$. 
\end{prop}

We refer to the Appendix for a proof. Although the Lipschitz assumption is necessary for the theoretical proof, we empirically show in the experimental section that the ANNR is well-behaved even for highly discontinuous functions such as characteristic functions of geometrically articulated domains (see Sec.~\ref{annr-advantages}).

The proof of Proposition \ref{convergence} remains valid when $s_t$ and the corresponding maximum is computed w.r.t. a partial Delaunay triangulation coming from the approximation discussed in Sec. \ref{secapprox}. That is, the ANNR is guaranteed to halt even with the approximation procedure.

%% file: sections/experiments.tex
\section{Experiments}

We select DEFER as our primary baseline as it is the state-of-the-art method for active function approximation which is suitable for arbitrary functions \cite{defer}. As an ablation, we additionally compare with the non-active version of the NNR which samples datapoints \emph{uniformly} from $A$, denoted by nANNR. 



In all the experiments, the number of queries is referred to as $N$. For numerical validation, we rely on the standard score Mean Average Error 
$\textnormal{MAE} = \frac{1}{|P_{\textnormal{test}}|}\sum_{p \in P_{\textnormal{test}}}{\left|\widetilde{f}(p) - f(p)\right|}$, where $P_{\textnormal{test}}$ is a test set. We generate  $P_{\textnormal{test}}$ as an equally-spaced grid in $m=2$ dimensions and by uniformly sampling from $A$ when $m>2$ since exhaustive grid sampling is computationally infeasible.

\subsection{Hyperparameter $\lambda$}
\label{exp:lambda} As discussed in Sec.~\ref{method-main} and \ref{geomintuition}, the ANNR has a single hyperparameter -- the lifting coefficient $\lambda$ -- which governs a natural exploration-exploitation trade-off. This is graphically demonstrated in Fig.~\ref{fig:lambda}, where a higher $\lambda$ encourages querying in areas where $f$ varies the most. In practice, we suggest the following heuristic choice for $\lambda$, which we implement in our experiments. We select $\lambda$ proportional to the size of the domain and inversely proportional to the scale of the function, effectively bringing domain and codomain to the same scale to balance exploration and exploitation of the function: $\lambda = \frac{\textnormal{Vol}(A)}{\max{{f}} - \min{{f}}}$. Here, $\max$ and $\min$ can be estimated from prior knowledge or, more concretely, directly evaluated from the initial dataset $P_0$.

\begin{figure}[h]
\centering
    \begin{subfigure}[b]{.32\linewidth}
        \centering
        \subcaption*{$\lambda=0.1$}
        \includegraphics[width=\linewidth]{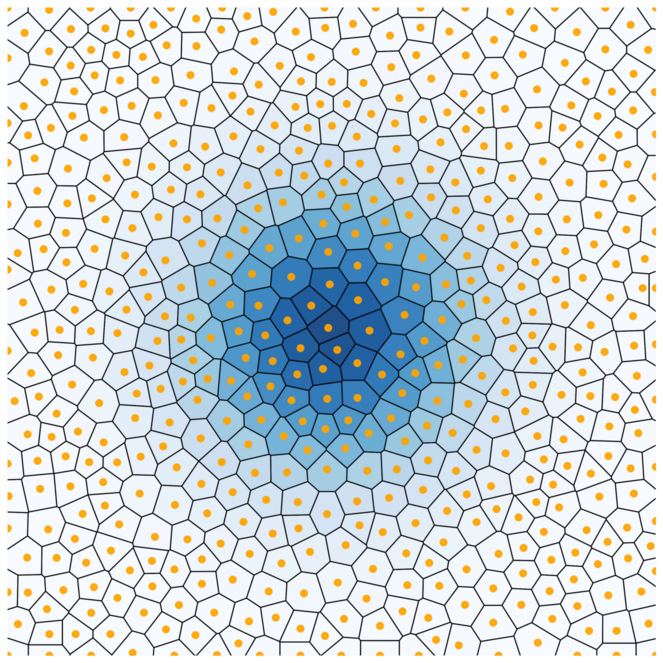}
    \end{subfigure}
    \begin{subfigure}[b]{.32\linewidth}
        \centering
        \subcaption*{$\lambda=1$}
        \includegraphics[width=\linewidth]{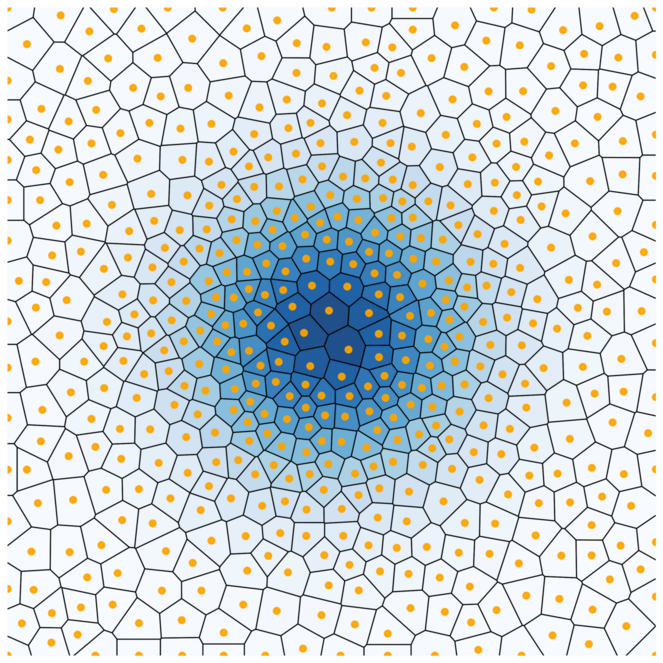}
    \end{subfigure}
    \begin{subfigure}[b]{.32\linewidth}
        \centering
        \subcaption*{$\lambda=10$}
        \includegraphics[width=\linewidth]{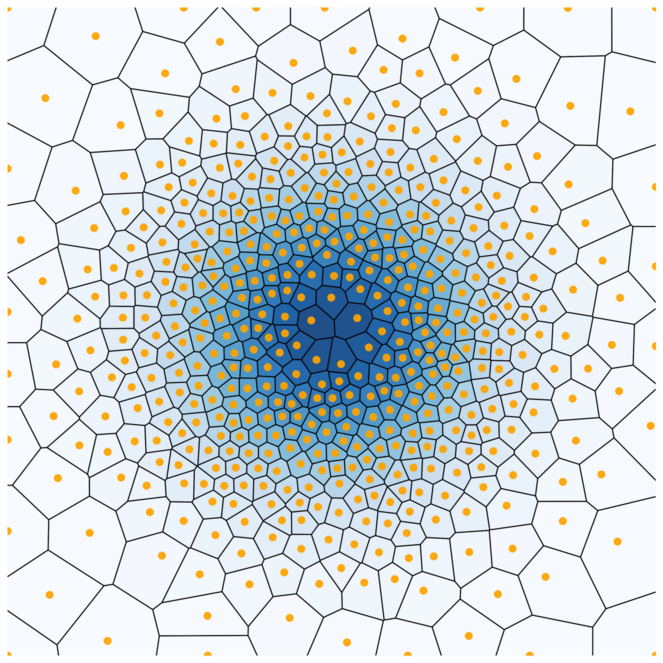}
    \end{subfigure}
    
    \caption{ANNR approximation of a normalized Gaussian with $\sigma^2=0.1$ for various values of $\lambda$ ($N = 500$).} 
    \label{fig:lambda}
\end{figure}


\subsection{Geometric Advantages of the ANNR}
\label{annr-advantages}
In this experiment, we consider characteristic functions with various supports and demonstrate that the ANNR is well suited for approximating the support shape. The ANNR  leverages the variable geometry of Voronoi cells, which results in a fine-grained approximation with a small number of queries. In contrast, DEFER relies on a rectangular partitioning of the space, which, as we show, can result in sub-optimal approximations. This feature also allows ANNR to approximate functions with arbitrary compact domains, in particular those with a small volume compared to $A$ (see the Appendix for an example). 

\begin{figure}[h]
\centering
    \begin{subfigure}[b]{.3\linewidth}
        \centering
        \subcaption*{DEFER}
        \includegraphics[width=\linewidth]{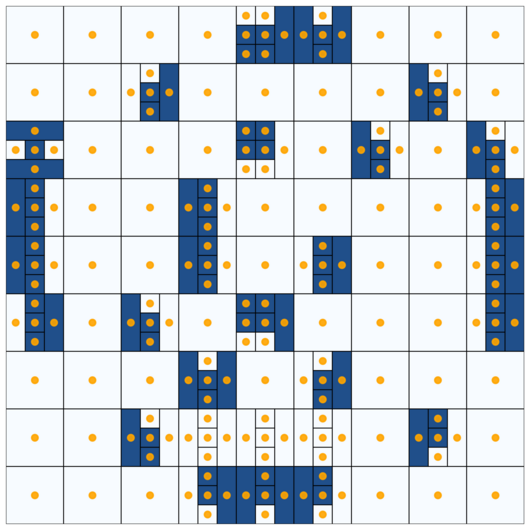}
    \end{subfigure}
    \hspace{\baselineskip}
    \begin{subfigure}[b]{.3\linewidth}
        \centering
        \subcaption*{ANNR}

        \includegraphics[width=\linewidth]{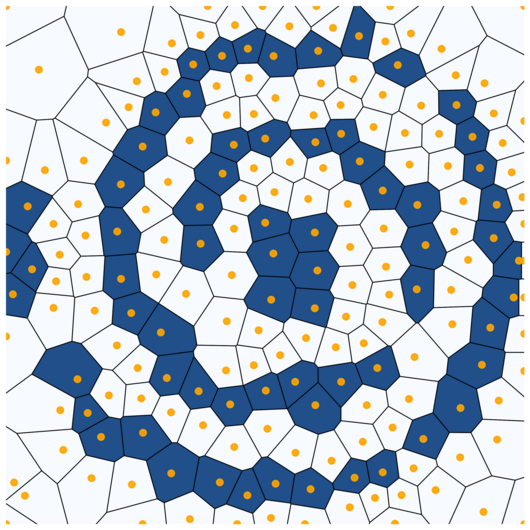}
    \end{subfigure}
    
    \begin{subfigure}[b]{.3\linewidth}
        \centering
        \vspace{0.5\baselineskip}
        \includegraphics[width=\linewidth]{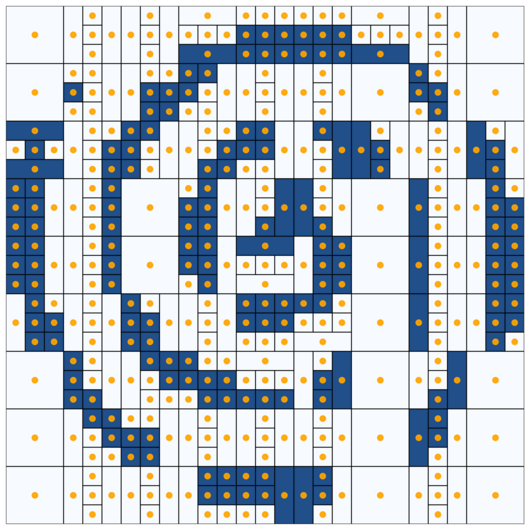}
    \end{subfigure}
    \hspace{\baselineskip}
    \begin{subfigure}[b]{.3\linewidth}
        \centering
        \vspace{0.5\baselineskip}
        \includegraphics[width=\linewidth]{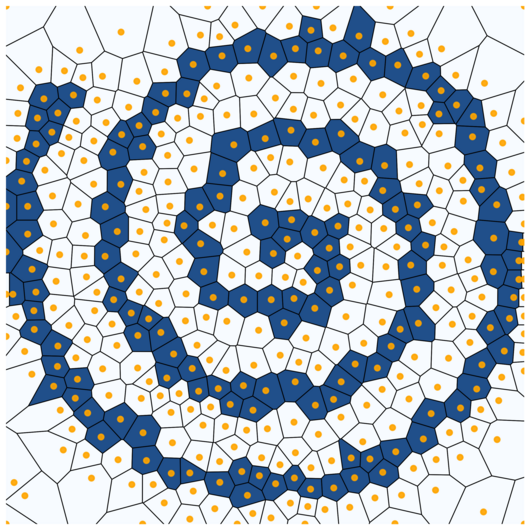}
    \end{subfigure}
    
    \caption{Approximation of a spiral characteristic function (see the Appendix for details). $N = 200$ (top), $N = 400$ (bottom).}
    \label{fig:spiral}
\end{figure}

We first consider the characteristic function of a spiral (Fig.~\ref{fig:spiral}). The ANNR displays a visibly better approximation and has a connected support already after $N=400$ queries. The support of DEFER is instead highly disconnected as it is unable to capture regions where the spiral is misaligned with the Cartesian axes due to its rectangular bias (see the Appendix for details).

\textbf{Rotational Equivariance.}
DEFER is sensitive to rotations of the domain due to its inherent bias towards Cartesian axes. In contrast, the ANNR allows to approximate shapes without any orientation bias since Voronoi cells are arbitrary polytopes. Fig.~\ref{fig:ellipse} demonstrates improved stability of the ANNR approximation with respect to rotations of $\mathbb{R}^2$ compared to DEFER (see the Appendix for details).

\begin{figure}[!tbh]
\centering
    \begin{subfigure}[b]{.3\linewidth}
        \centering
        \subcaption*{DEFER}

        \includegraphics[width=\linewidth]{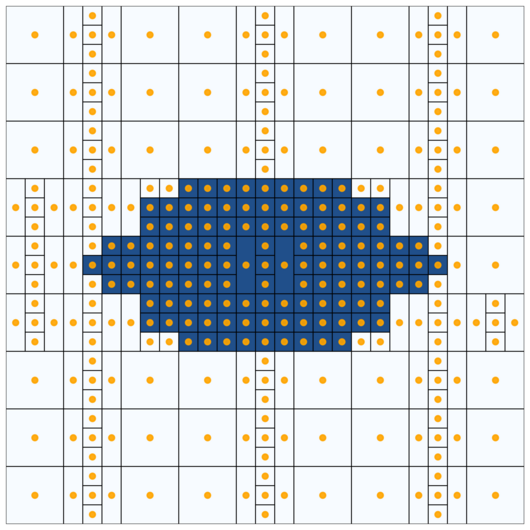}
    \end{subfigure}
    \hspace{\baselineskip}
    \begin{subfigure}[b]{.3\linewidth}
        \centering
        \subcaption*{ANNR}

        \includegraphics[width=\linewidth]{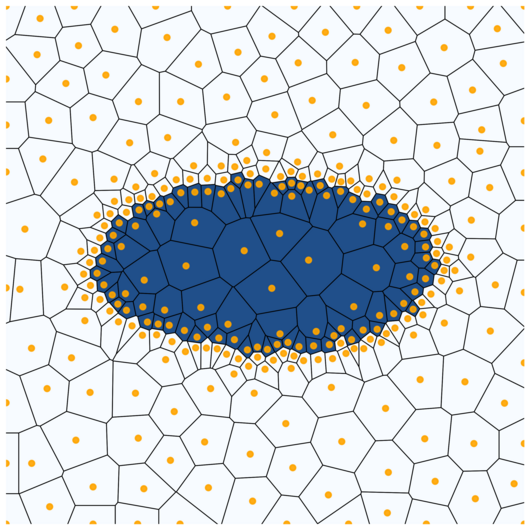}
    \end{subfigure}
    
    \begin{subfigure}[b]{.3\linewidth}
        \centering
        \vspace{0.5\baselineskip}
        \includegraphics[width=\linewidth]{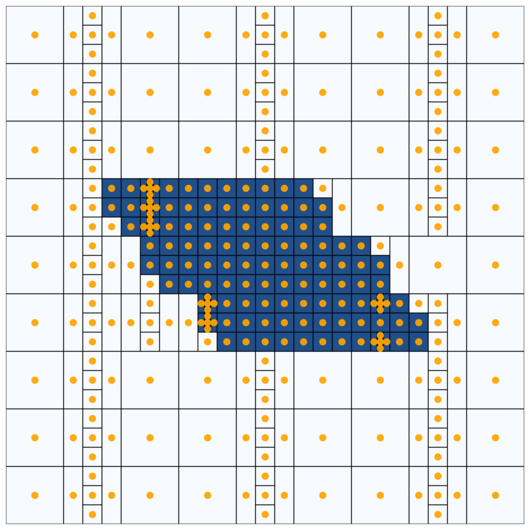}
    \end{subfigure}
    \hspace{\baselineskip}
    \begin{subfigure}[b]{.3\linewidth}
        \centering
        \vspace{0.5\baselineskip}
        \includegraphics[width=\linewidth]{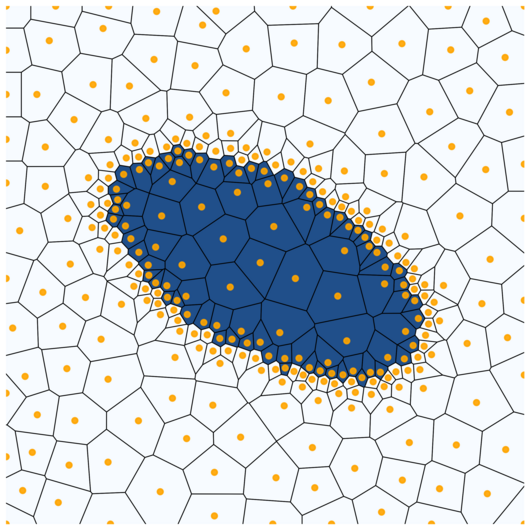}
    \end{subfigure}
    
    \caption{Approximation of an ellipse characteristic function $\mathbf{1}_{x^2 + 4y^2 \le 1}$ with $0^\circ$ (top) and $30^\circ$ (bottom) rotations of the domain ($N=300$).}
    \label{fig:ellipse}
\end{figure}

\textbf{Curse of Dimensionality.}
The bias of DEFER towards rectangular geometry can potentially affect the quality of approximation in high-dimensions due to the increasing \textit{spikiness} of rectangles~\cite{devi2016conceptualizing, blum2020foundations}. To demonstrate this, we consider a characteristic function $\mathbf{1}_{\| x\| \leq 1}$ of a unit ball in $\mathbb{R}^6$ (with $A=[-2,2]^6$). Fig.~\ref{fig:ball10d} reports the score (Fig.~\ref{fig:ball10d-mae}) and query point distribution (Fig.~\ref{fig:ball10d-queries}) with respect to the norm of test and query points respectively.

\begin{figure}[!tbh]
    \centering
    \begin{subfigure}{.45\linewidth}
    \centering
    \includegraphics[width=\linewidth]{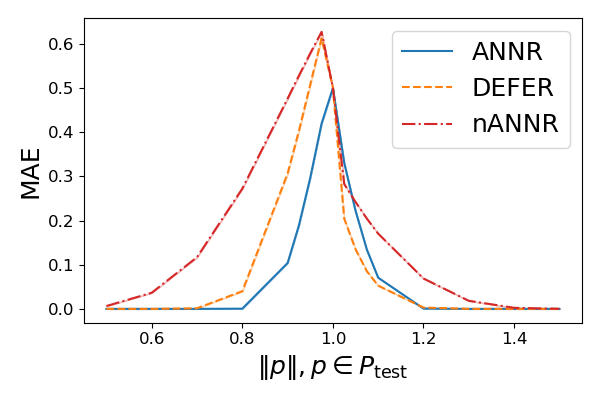}
    \subcaption{\centering MAE scores for test sets at different norms.}
    \label{fig:ball10d-mae}
    \end{subfigure}
    \begin{subfigure}{.45\linewidth}
    \centering
    \includegraphics[width=\linewidth]{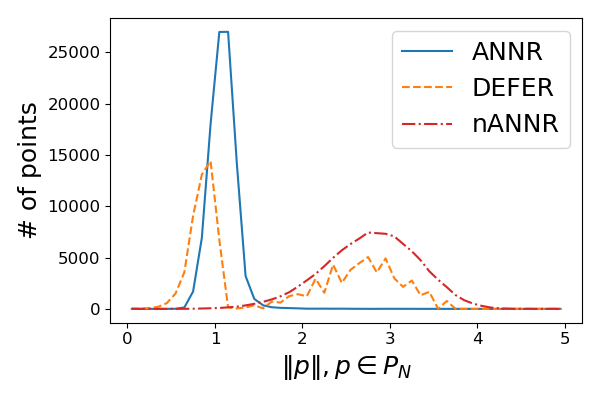}
    \subcaption{\centering Norm distribution of the queried points.}
    \label{fig:ball10d-queries}
    \end{subfigure}
    \caption{Analysis of the $6$-dimensional unit ball approximations based on the norm of data ($N=10^5$ and $\lvert P_{\textnormal{test}} \rvert = 10^7$ per norm).}
    \label{fig:ball10d}
\end{figure}

Fig.~\ref{fig:ball10d-mae} shows that the MAE scores decrease when the test points are sampled away from the discontinuity of $f$ i.e., the boundary of the ball. ANNR's score is significantly lower inside the ball ($\| p\|<1$, $p \in P_{\textnormal{test}}$). At the same time, Fig.~\ref{fig:ball10d-queries} shows that all the queries of the ANNR are queried in the proximity of the boundary, while DEFER's query norm distribution is shifted towards the uniform one (nANNR). 
 
\subsection{Performance Comparison}





In this section we compare the performance of the ANNR with DEFER and the non-active sampler nANNR for articulated functions including parameter estimation of gravitational waves and exploration of the latent space of a deep generative model. The scores are presented in Table \ref{tab:mae-performance}. For the ANNR and the nANNR, we report averages and standard deviations over $10$ runs with different random initializations (in contrast, DEFER has no stochasticity). We set $N=1000$ for the two-dimensional experiment (latent manifold exploration) and $N=10^5$ when $m>2$. Our method outperforms the baseline in terms of MAE by a significant margin.

\begin{table}[!tbh]   
\centering
\caption{Performance Comparison (MAE).}
\label{tab:mae-performance}
\begin{small}  
\begin{tabular}{cccc}
\toprule
& ANNR & DEFER & nANNR \\
\midrule
\begin{tabular}[c]{@{}c@{}}Gravitational waves\\parameter estimation\end{tabular}         & \begin{tabular}[c]{@{}c@{}}$0.4710$\\ $\pm 0.0232$\end{tabular}  & $0.5309$  & \begin{tabular}[c]{@{}c@{}}$0.5317$\\ $\pm 0.2175$\end{tabular}  \\
\rule{0pt}{15pt}   
\begin{tabular}[c]{@{}c@{}}Latent space\\ volume density\end{tabular} & \begin{tabular}[c]{@{}c@{}}$47.0509$\\ $\pm 0.4975$\end{tabular} & $50.5073$ & \begin{tabular}[c]{@{}c@{}}$54.6290$\\ $\pm 0.8010$\end{tabular} \\
\bottomrule
\end{tabular}
\end{small}
\end{table}

\textbf{Gravitational Waves.}
A large area of application for active function approximation is astrophysics and, in particular, gravitational wave parameter inference problems \cite{dynesty, bilby}.
Observation of gravitational waves are rare, and the magnitude of the effect is low, implying scarcity of collected data. With limited observations and dozens of parameters describing the gravitational event, the problem of parameter estimation is commonly solved by Bayesian inference methods approximating a simulated log-likelihood function on the parameters \cite{bilby, gw-bayesian}.


We use the same 6-dimensional formulation of parameter inference as described in \cite{defer}, and refer the reader to the Appendix for a detailed description of parameters. In addition to the original setup, we rescale the function domain to a unit hypercube $A =[0,1]^6$ for convenience and multiply $f$ by a factor of $e^{8000}$ for better numerical stability. Lastly, in order to deal with practical unboundedness of the density function, we perform an adaptive clipping of extensively sharp volumes. The details of the clipping are available in the Appendix. 

Apart from the MAE comparison in Table~\ref{tab:mae-performance}, we present a visualization of the function approximations in Fig.~\ref{fig:gw-marginals} of a single run. The figure displays histograms over marginal distributions over a set of two-dimensional slices of the domain. Each marginalization was performed via Monte-Carlo sampling with $10^5$ samples per bin. The figure shows the visual closeness of ANNR marginalizations to the original distribution compared to the baseline. Marginals over all 15 pairs of parameters are available in the Appendix.

    
\newcommand{\rotated}[1]{\raisebox{-.5\normalbaselineskip}[0pt][0pt]{\rotatebox[origin=c]{90}{\hspace{10pt}\footnotesize #1}}}
    
\begin{figure}[tb!]
    \setlength{\tabcolsep}{0pt}
    \newcommand{\w}{.22\linewidth}
    \renewcommand{\arraystretch}{0}
    \centering
    \begin{tabular}{p{0.07\linewidth}ccc}
        & {\footnotesize $(d_L, \theta)$} & {\footnotesize $(d_L, \psi)$} & {\footnotesize $(a_1, a_2)$}\\[0pt]

        \rotated{Ground Truth} & \includegraphics[align=c, width=\w]{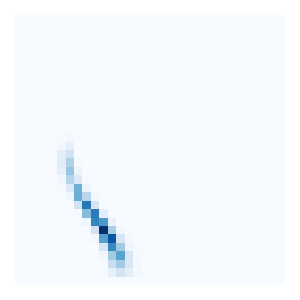} & \includegraphics[align=c, width=\w]{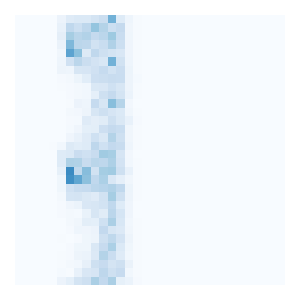} & \includegraphics[align=c, width=\w]{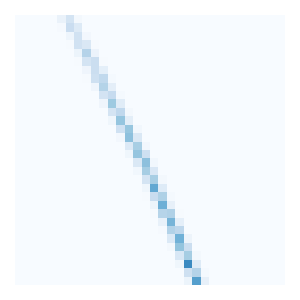}\\[0pt]
        
        \rotated{DEFER} & \includegraphics[align=c, width=\w]{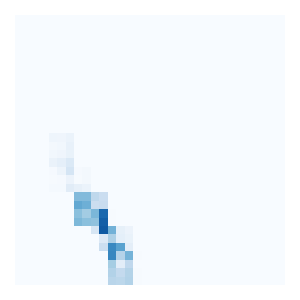} & \includegraphics[align=c, width=\w]{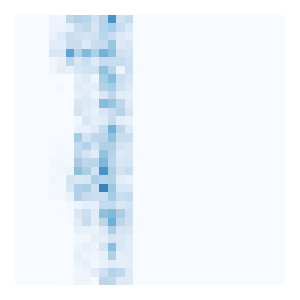} & \includegraphics[align=c, width=\w]{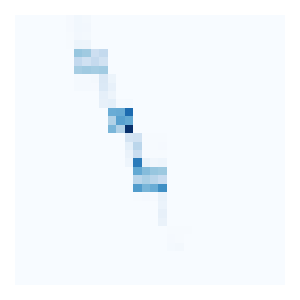}\\[0pt]
        
        \rotated{ANNR} & \includegraphics[align=c, width=\w]{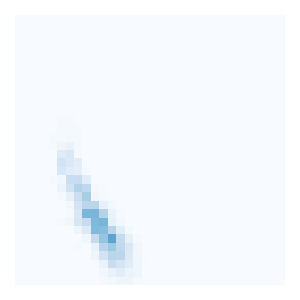} & \includegraphics[align=c, width=\w]{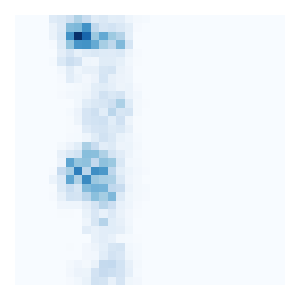} & \includegraphics[align=c, width=\w]{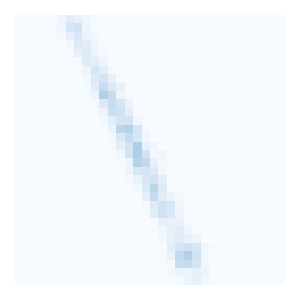}\\[0pt]
    \end{tabular}
    \caption{Marginal distributions for gravitational waves over selected two-dimensional slices. Each slice is identified by a pair of parameters named as in \cite{defer}.}
    \label{fig:gw-marginals}
\end{figure}

\textbf{Latent Manifold Exploration.}
Our last experiment addresses a manifold exploration task on the latent space of a deep generative model. The generative process enables to query an arbitrary point in the latent space and thus suits the framework of the ANNR and DEFER. Moreover, decoding the queried points allows to interpret and semantically evaluate the exploration procedure. To this end, we deploy a generative model $\varphi: \ \mathcal{Z} \rightarrow \mathcal{X}$ that maps a prior distribution on the latent space $\mathcal{Z} = \mathbb{R}^m$ to the data distribution on $\mathcal{X} = \mathbb{R}^n$ and regard it as (the probabilistic analogue of) a parametrization of the data manifold. The `mass' distribution of data on $\mathcal{Z}$ is represented by the induced \emph{volume density} $f$ i.e., the density of the volume form of the Riemannian metric induced by $\varphi$ on $\mathcal{Z}$ via pull-back \cite{arvanitidis2017latent, arvanitidis2020geometrically}. Such volume density is concretely expressed as $f(z) =  \sqrt{ \textnormal{det}\left( J_\varphi(z)^T J_\varphi(z)  \right)}$ where $J_\varphi(z)$ is the Jacobian matrix of $\varphi$ at $z \in \mathcal{Z}$. Intuitively, $f$ can be seen as a 'fuzzy' characteristic function of the latent manifold and an active function approximator for $f$ can be interpreted as progressively exploring such manifold.

\begin{figure}[tb!]
\centering
    \begin{subfigure}[b]{.4\linewidth}
        \centering
    \subcaption*{$N=100$}
        \includegraphics[width=\linewidth]{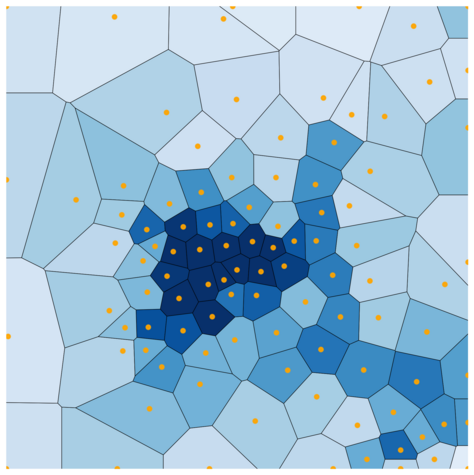}
    \end{subfigure}
    \begin{subfigure}[b]{.4\linewidth}
        \centering
           \subcaption*{$N=500$}

        \includegraphics[width=\linewidth]{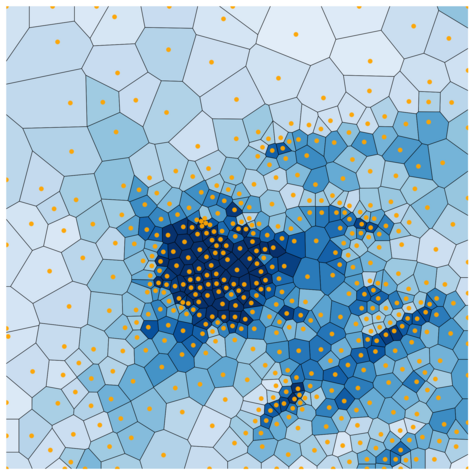}
    \end{subfigure}
    
    \begin{subfigure}[b]{.4\linewidth}
    \centering
    \subcaption*{$N=1000$}
    \includegraphics[width=\linewidth]{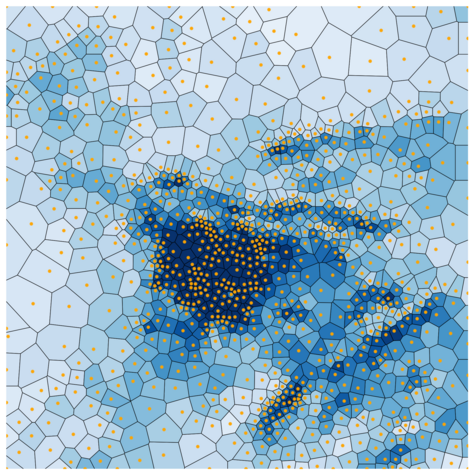}
    \end{subfigure}
    \begin{subfigure}[b]{.4\linewidth}
    \centering
       \subcaption*{Ground Truth}

    \includegraphics[width=\linewidth]{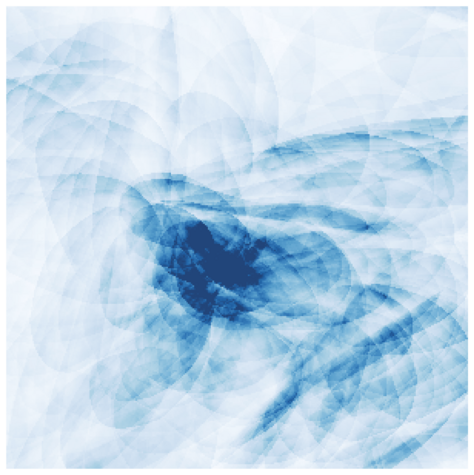}
    \end{subfigure}
    
     \vspace{0.5\baselineskip}
    \begin{subfigure}[b]{.98\linewidth}
    \centering
       \subcaption*{$490 < t \leq 500$}

    \includegraphics[width=\linewidth]{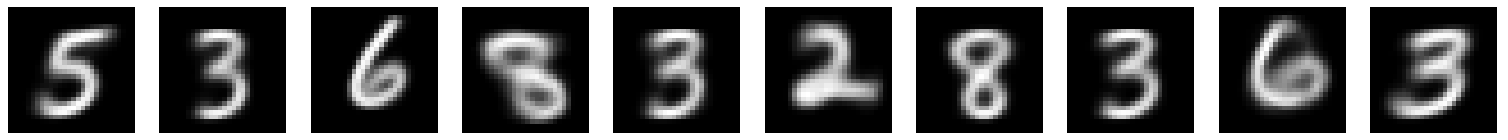}
    \end{subfigure}
    
    \vspace{0.5\baselineskip}
    \begin{subfigure}[b]{.98\linewidth}
    \centering
       \subcaption*{$990 < t \leq 1000$}

    \includegraphics[width=\linewidth]{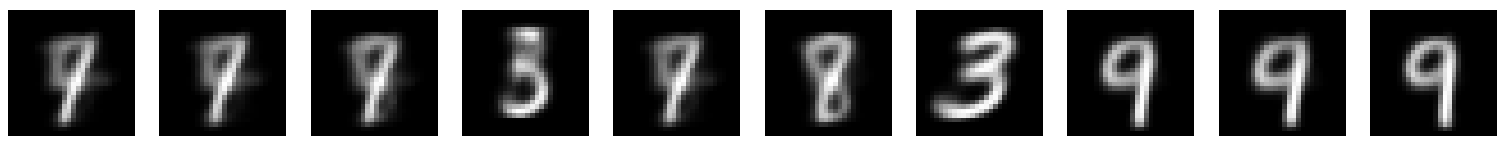}
    \end{subfigure}

    \caption{\textbf{Top}: approximation of the latent space volume density of a deep generative model. \textbf{Bottom}: images decoded from the queried latent points $p_t$ for two intervals of~$t$. }
    \label{fig:latent_exploration}
\end{figure}

We train $\varphi$ as (the decoder of) a Variational Autoencoder (VAE, \citet{kingma2013auto}) on the MNIST dataset \cite{deng2012mnist} of gray-scale images of hand-written digits ($n=784$). We deploy a two-dimensional latent space ($m=2$) with a standard Gaussian prior. Despite its low dimensionality, such a latent space is sufficient for generation of quality MNIST images \cite{kingma2013auto} and enables direct visualization. For better parametrization quality, the latent prior is additionally encouraged by a hyperparameter $\beta = 2$ multiplying the corresponding ELBO loss term \cite{higgins2016beta}.

Fig.~\ref{fig:latent_exploration} displays the progressive approximation of $f$ as well as images decoded from queried points. As can be seen from the latter, around $t=500$ the class of $p_t$ (digit) is different at each step $t$ and the ANNR is thus covering a wide semantic range of data. In contrast, around $t=1000$ the decoded images are similar, indicating a phase close to convergence in which the approximation is refined locally.

\textbf{Runtime Comparison.}
The overall flexibility of the geometric framework utilized by the ANNR comes with a certain computational cost, which highlights a tradeoff between effectiveness and efficiency of the methods. The computational cost of the ANNR is largely mitigated by the Delaunay approximation the method relies on. In addition, the runtime of the ANNR can be heavily controlled by adjusting the desired Markov Chain sampling precision and increasing the number of available parallel threads.

We provide a runtime comparison between the methods on experiments discussed earlier in this section in Table \ref{tab:runtime-comparison}. Each number represents the average runtime for a single experimental run. All experiments are performed on CPU Ryzen 9 5950X 16-Core. We note that in case of the ANNR single queries still take milliseconds to compute, which is reasonable for real-life applications with high function evaluation cost.

\begin{table}[!tbh]   
\centering
\caption{Runtime Comparison.}
\label{tab:runtime-comparison}
\begin{small}  
\begin{tabular}{cccc}
\toprule
& ANNR & DEFER \\
\midrule
\begin{tabular}[c]{@{}c@{}}Gravitational waves\\parameter estimation\end{tabular}         & $1347$~s & $186$~s  \\
\rule{0pt}{15pt}   
\begin{tabular}[c]{@{}c@{}}Latent space\\ volume density\end{tabular} & $220$~ms & $90$~ms \\
\bottomrule
\end{tabular}
\end{small}
\end{table}


%% file: sections/discussion.tex
\section{Conclusion and Future Work}
In this work we introduced the ANNR -- an adaptation of nearest neighbor regression to an active setting. We provided a computationally  efficient implementation of the ANNR as well as theoretical halting guarantees. 
Our empirical investigations have shown that the ANNR outperforms the state-of-the-art active function approximator called DEFER. 

An interesting line for future investigation lies in designing active querying strategies for other higher-dimensional function approximators from the existing literature. Examples include the piecewise linear Delaunay interpolator \cite{omohundro1989delaunay} or the tensor product of cubic splines \cite{lalescu2009two}. 

\section{ACKNOWLEDGEMENTS}
This work was supported by the Swedish Research Council, the Knut and Alice
Wallenberg Foundation and the European Research Council (ERC-BIRD-884807).

%% file: sections/appendix.tex
\section{Proofs of Theoretical Results}
\label{appendix-proofs}
In this section we provide proofs for the theoretical results presented in the main body of the paper. We start by proving the result from Sec.~\ref{interpretation}.

\begin{prop}\label{appendix-interpretation}
There exists a constant $C > 0$ such that for every smooth function $f : \ \Omega \rightarrow \mathbb{R}$ the following inequality holds: 
\begin{equation}\label{appendix-inequal}
\log \textnormal{Vol}(\Gamma_f)  \geq C \|f - f_\Omega \|_2^2 + \log \textnormal{Vol}(\Omega) + o(\| g_f \|^2).
\end{equation}
\end{prop}

\begin{proof}
The volume of $\Gamma_f$ can be expressed as 
\begin{equation}
\textnormal{Vol}(\Gamma_f) = \int_\Omega \sqrt{\textnormal{det}(1 + g_f)}
\end{equation}
and thus by Jensen inequality we get:  
\begin{equation}\label{appendix-jensen}
    \log \frac{\textnormal{Vol}(\Gamma_f)}{\textnormal{Vol}(\Omega)} \geq \frac{1}{2 \textnormal{Vol}(\Omega)}\int_\Omega \log \textnormal{det}(1 + g_f).
\end{equation}
Since by general properties of matrices $\log \textnormal{det}(g_f) = \textnormal{tr}(\log g_f)$ and $\log(1 + g_f) = g_f + o(\| g_f \|^2 )$ where $\| \cdot \|$ denotes the standard matrix norm, the right hand side of Eq.~\ref{appendix-jensen} reduces to

\begin{equation}
    C_1 \int_\Omega \textnormal{tr}(g_f) + o(\| g_f\|^2) = C_1 \| \nabla_f \|_2^2 + o(\| g_f \|^2).
\end{equation}
Finally, by Poinar\'e-Wirtinger inequality \cite{brezis2010functional} $\| \nabla_f \|_2^2 \geq C_2 \| f - f_\Omega\|_2^2$ for some constant $C_2 >0$, which concludes the proof. 
\end{proof}

We now prove the halting guarantee from Sec.~\ref{convergence}. First, we provide the following result concerning high-dimensional Euclidean geometry which will be necessary for the main proof. Recall that $A \subseteq \mathbb{R}^m$ is a convex compact set. 

\begin{lemm}\label{appendix-euclid}
For each $\delta > 0$ there exists an $\eta > 0$ such that for each $n$-dimensional simplex $\sigma \subseteq A$ if $\textnormal{Vol}(\sigma) > \delta$ then $d(v, x) > \eta$ for each vertex $v$ of $\sigma$ and each point $x$ of the segment connecting the baricenter and the circumcenter of $\sigma$. 
\end{lemm}

\begin{proof}
By comparing the volume of a simplex with the volume of its circumsphere, there exists $\eta_1> 0$ such that for any simplex $\sigma$ if  $\textnormal{Vol}(\sigma) > \delta$ then $R_\sigma > \eta_1$, where $R_\sigma$ denotes the radius of the circumsphere of $\sigma$. There also exists a $\eta_2> 0$ such that for any simplex $\sigma$ if  $\textnormal{Vol}(\sigma) > \delta$ then all the heights (i.e., the segments passing through the vertices and orthogonal to the opposite faces) of $\sigma$ are greater than $\eta_2$. To see this, note that otherwise there would exist triangles of arbitrarily small heights and thus, by the volume constraint, with faces of arbitrarily large volume, contradicting the compactness of $A$. Now, given a simplex $\sigma$ and a vertex $v$ of $\sigma$, the segment connecting $v$ to its opposite face and containing the barycenter $b$ of $\sigma$ is shorter than the height passing through $v$ because of orthogonality of the latter. The barycenter separates that segment into a portion of $\frac{m-1}{m}$ of it containing $v$. 
Thus, we have $d(v, b) > \frac{m-1}{m}\eta_2$. 

Consider $\eta_3 = \min \{\eta_1, \frac{m-1}{m}\eta_2 \}$. For any simplex $\sigma$ with circumenter $c$ and barycenter $b$ we then have $d(v,b) > \eta_3$ and $d(v, c) > \eta_3$ for every vertex $v$ of $\sigma$. Since $(b - v ) \cdot (c-v) > 0$, we have that $d(v, x) > \frac{\sqrt{2}}{2}\eta_3 = \eta$ for each $x$ in the segment connecting $b$ and $c$, as desired.  
\end{proof}

We now prove the main result from Sec.~\ref{convergence}.

\begin{prop}\label{appendix-lowerlimit}
 Assume that the ground-truth function $f$ is Lipschitz and let $s_t = \textnormal{max}_{\sigma \in \dela_{P_t}} \textnormal{Vol}(\hat{\sigma})$. Then Alg.~\ref{alg:nnr} always halts for any $\varepsilon$ or, in other words,  $\underline{\textnormal{lim}}_{t \to \infty} s_t = 0$. 
\end{prop}
\begin{proof}
Suppose by contradiction that there exists $\varepsilon > 0$ such that $s_t > \varepsilon$ for all sufficiently large $t$. Since $f$ (and thus $\lambda f$) is Lipschitz, there exists $\delta > 0$ such that for any simplex $\sigma$ if $\textnormal{Vol}(\hat{\sigma}) > \varepsilon$ then $\textnormal{Vol}(\sigma) > \delta$. 

Let us prove that there can not be infinite queries in the interior of the bounding region $A$. By comparing the volume of a simplex with the volume of its circumsphere, there exists $\gamma> 0$ such that for any simplex $\sigma$ if  $\textnormal{Vol}(\sigma) > \delta$ then $R_\sigma > \gamma$, where $R_\sigma$ denotes the radius of the circumsphere of $\sigma$. When a point $p_{t +1} = \circu ( \overline{\sigma})$ is queried, since by hypothesis $s_t = \textnormal{Vol}(\hat{\overline{\sigma}}) > \varepsilon$ we have $R_{\overline{\sigma}} > \gamma$. But then $d(p_{t +1}, q) > \gamma$ for all $q \in P_t$ since no points in $P_t$ are contained inside the circumsphere of the Delaunay simplex $\overline{\sigma}$. If infinite points in the interior of $A$ are queried, one gets an infinite sequence in $A$ whose pairwise distances are all greater than $\gamma$. This is impossible since $A$ is compact. 

As a consequence, the queried points $p_{t+1}$ belong to the boundary of $A$ for $t \gg 0$ and are obtained via the method described in Sec.~\ref{boundingsection}. By Lemma~\ref{appendix-euclid} there exists $\eta > 0$ such that $d(v, p_{t +1}) > \eta$ for $ t \gg 0$ where $v$ is any vertex of the simplex  $\overline{\sigma}$ whose lifting has the largest volume at step $t$. Since $\overline{\sigma}$ is a Delaunay simplex it holds that $d(q, p_{t +1}) > \eta$ for every $q \in P_t$ and $t \gg 0$, which again contradicts the compactness of $A$. 
\end{proof}

\section{Extended Experiments}
\subsection{Spiral}
\label{appendix-support}
\begin{figure}[h]
    \centering
    \begin{subfigure}[b]{.2\linewidth}
        \centering
        \includegraphics[width=\linewidth]{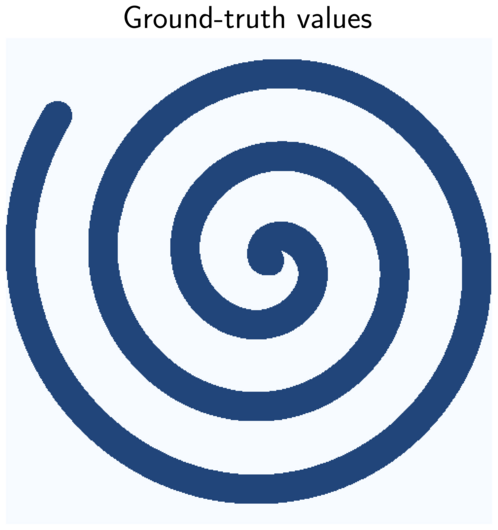}
    \end{subfigure}
    \hspace{\baselineskip}
    \begin{subfigure}[b]{.3\linewidth}
        \centering
        \includegraphics[width=\linewidth]{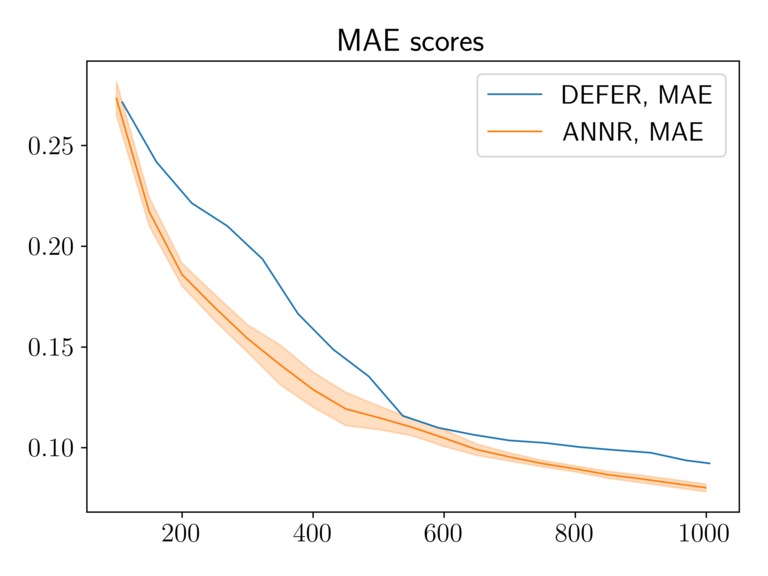}
    \end{subfigure}
    \hspace{\baselineskip}
    \begin{subfigure}[b]{.3\linewidth}
        \centering
        \includegraphics[width=\linewidth]{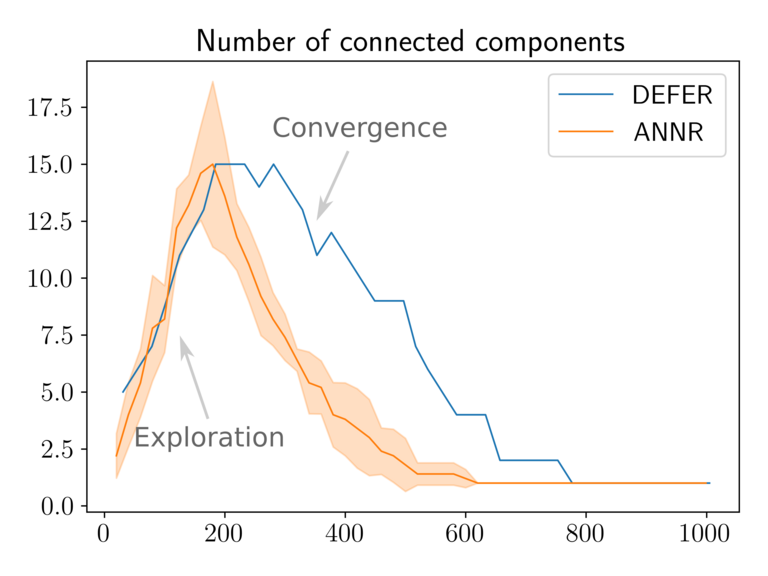}
    \end{subfigure}
    
    \caption{Ground-truth plot of a spiral characteristic function and performance comparison between the ANNR and DEFER as the number of queries increases. Scores are averaged over 5 runs.}
    \label{fig:spiral-app}
\end{figure}

\subsection{Arbitrary Domain}
\label{appendix-arbitrary}
In many cases the actual domain of a function is not the entire $\mathbb{R}^m$, but rather some compact subspace. In many optimization algorithms, the function domain is artificially extended to a bounding volume (usually a bounding box) containing it, assuming some zero value outside of the original function domain. In higher dimensions, this significantly increases the volume to explore~\cite{devi2016conceptualizing}. The ANNR takes advantage of the flexibility of Delaunay partitioning to restrict new queries to the (boundary of the) function domain. Fig.~\ref{fig:narrow} illustrates an approximation of $f(x) = \| x \|$ restricted to the intersection of two circles (the volume of the domain is only $\sim 1/30$ of the bounding box volume), which otherwise can be very inefficiently approximated by its bounding box. 

\begin{figure}[h]
    \centering
    \begin{subfigure}[b]{.19\linewidth}
        \centering
        \includegraphics[width=\linewidth]{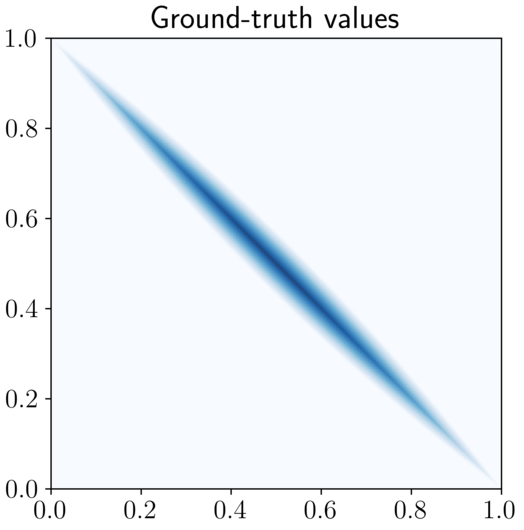}
    \end{subfigure}
    \hspace{\baselineskip}
    \begin{subfigure}[b]{.18\linewidth}
        \centering
        \includegraphics[width=\linewidth]{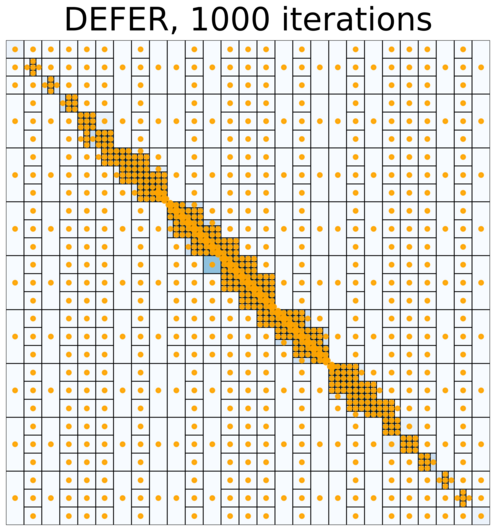}
    \end{subfigure}
    \hspace{\baselineskip}
    \begin{subfigure}[b]{.18\linewidth}
        \centering
        \includegraphics[width=\linewidth]{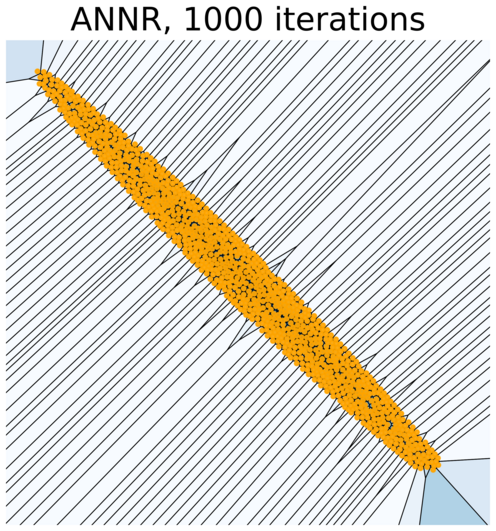}
    \end{subfigure}
    \hspace{\baselineskip}
    \begin{subfigure}[b]{.28\linewidth}
        \centering
        \includegraphics[width=\linewidth]{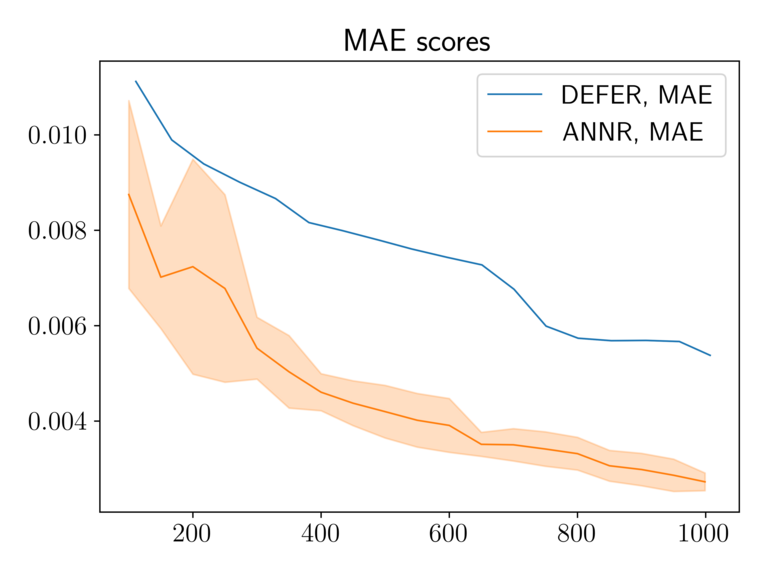}
    \end{subfigure}
    
    \caption{Approximation of a distance function defined on an intersection of two circles with radius $5$ and centered at (-3,~-3) and (4,~4). }
    \label{fig:narrow}
\end{figure}

\newpage

\subsection{Rotational Equivariance}
\label{appendix-rotation}
\begin{figure*}[tbh!]
    \hspace{2.3\baselineskip}
    \centering
    \begin{subfigure}[b]{.12\linewidth}
        \centering
        \subcaption*{\centering $0^{\circ}$}
    \end{subfigure}
    \begin{subfigure}[b]{.12\linewidth}
        \centering
        \subcaption*{\centering $10^{\circ}$}
    \end{subfigure}
    \begin{subfigure}[b]{.12\linewidth}
        \centering
        \subcaption*{\centering $20^{\circ}$}
    \end{subfigure}
    \begin{subfigure}[b]{.12\linewidth}
        \centering
        \subcaption*{\centering $30^{\circ}$}
    \end{subfigure}
    \begin{subfigure}[b]{.12\linewidth}
        \centering
        \subcaption*{\centering $40^{\circ}$}
    \end{subfigure}
    
        \begin{subfigure}[b]{.05\linewidth}
        \centering
        \subcaption*{\rotatebox[origin=c]{90}{\hspace{-5pt}DEFER}}
        \vspace{1.25\baselineskip}
    \end{subfigure}
    \centering
    \begin{subfigure}[b]{.12\linewidth}
        \centering
        \includegraphics[width=\linewidth]{images/ellipse/ellipse-defer-0.png}
    \end{subfigure}
    \begin{subfigure}[b]{.12\linewidth}
        \centering
        \includegraphics[width=\linewidth]{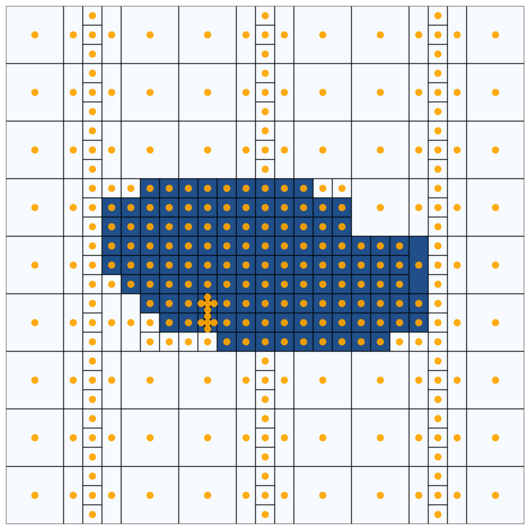}
    \end{subfigure}
    \begin{subfigure}[b]{.12\linewidth}
        \centering
        \includegraphics[width=\linewidth]{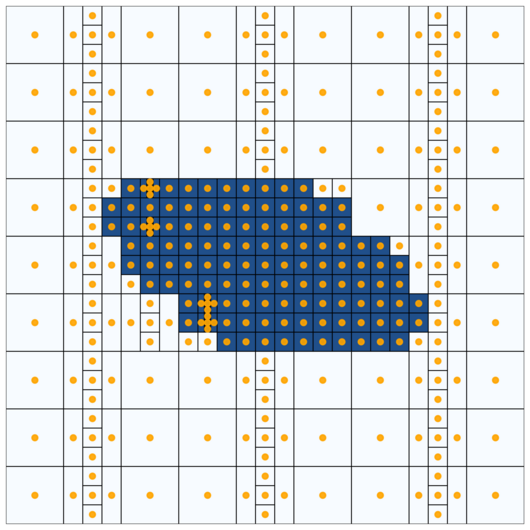}
    \end{subfigure}
    \begin{subfigure}[b]{.12\linewidth}
        \centering
        \includegraphics[width=\linewidth]{images/ellipse/ellipse-defer-30.png}
    \end{subfigure}
    \begin{subfigure}[b]{.12\linewidth}
        \centering
        \includegraphics[width=\linewidth]{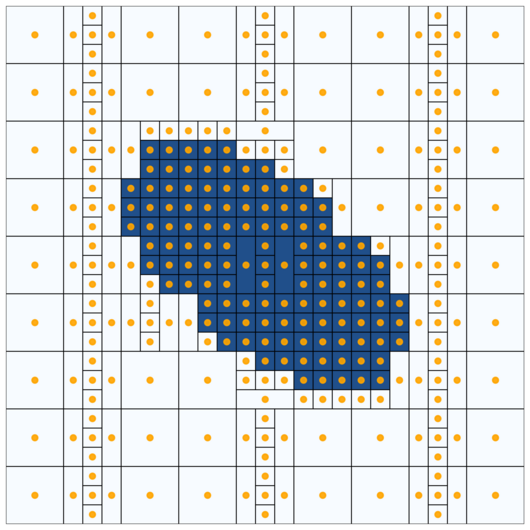}
    \end{subfigure}
    
    \begin{subfigure}[b]{.05\linewidth}
        \centering
        \subcaption*{\rotatebox[origin=c]{90}{\hspace{-5pt}ANNR}}
        \vspace{1.4\baselineskip}
    \end{subfigure}
    \centering
    \begin{subfigure}[b]{.12\linewidth}
        \centering
        \includegraphics[width=\linewidth]{images/ellipse/ellipse-annr-0.png}
    \end{subfigure}
    \begin{subfigure}[b]{.12\linewidth}
        \centering
        \includegraphics[width=\linewidth]{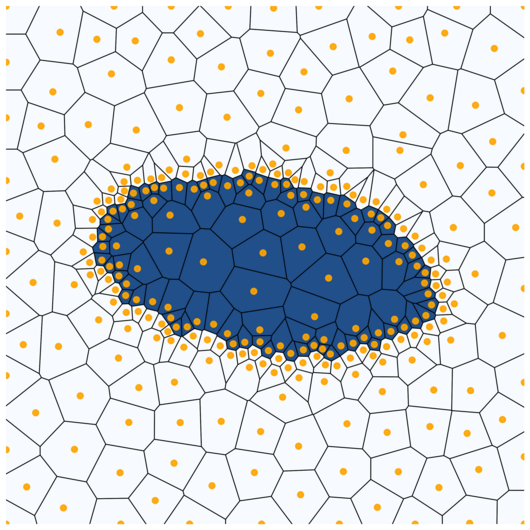}
    \end{subfigure}
    \begin{subfigure}[b]{.12\linewidth}
        \centering
        \includegraphics[width=\linewidth]{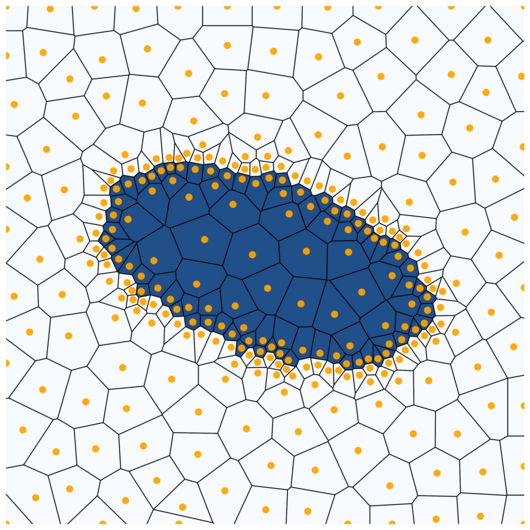}
    \end{subfigure}
    \begin{subfigure}[b]{.12\linewidth}
        \centering
        \includegraphics[width=\linewidth]{images/ellipse/ellipse-annr-30.png}
    \end{subfigure}
    \begin{subfigure}[b]{.12\linewidth}
        \centering
        \includegraphics[width=\linewidth]{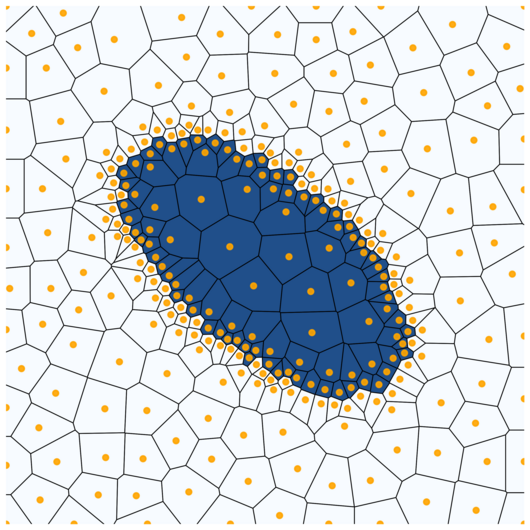}
    \end{subfigure}

       \includegraphics[width=.3\linewidth]{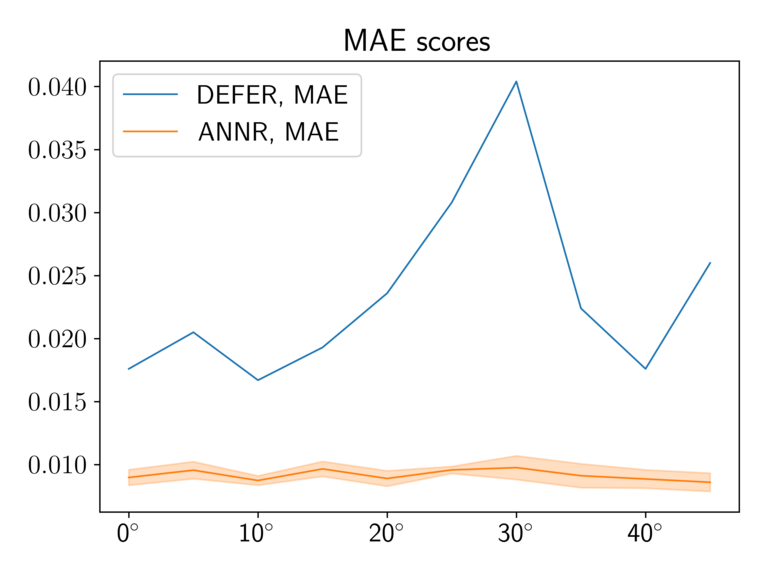}

    \caption{\textbf{Top}: approximation of an ellipse characteristic function $\mathbf{1}_{x^2 + 4y^2 \le 1}$ by DEFER and the ANNR with various rotations of the domain ($N=300$).
    \textbf{Bottom}: performance comparison as the angle varies. }
    \label{fig:ellipse-app}
\end{figure*}

\subsection{Gravitational Waves}
\label{appendix-gravitational}
\textbf{Data preprocessing.}
We follow the initial setup of the problem described in \cite{defer}. Namely, the function is constructed from a gravitational wave example provided \href{https://git.ligo.org/lscsoft/bilby/-/blob/c14012d6fe4f893d8b086d543ffd353dcef490b7/examples/gw_examples/injection_examples/fast_tutorial.py}{tutorial}. The six parameters of a gravitational wave generated by a binary black hole that are inferred by the model are: the luminosity distance $d_L$, the inclination angle $\iota$, the polarization angle $\psi$, the phase $\phi_c$ relative to a reference time and two spin magnitudes $a_1$ and $a_2$. For more details about the nature of the data we refer to \cite{bilby}.

Evaluation of the logarithm of the approximated likelihood on several uniformly selected points in the domain yields values distributed around $-8000$ (the simulator provides log-likelihood). Since our aim is to approximate the original likelihood function, we add $8000$ to all log-density evaluations and then exponentiate the values, effectively multiplying the likelihood function by $e^{8000}$.
This operation brings the majority of output values of the underlying function close to single digits, stabilizing the computations. Such scaling does not affect the baseline as DEFER is invariant to such transformations, and multiplicatively affects our choice of the lifting parameter.

\textbf{Volume clipping.}
While the approximated density function may appear to take relatively low values over the large part of the domain, some of the concentrated areas may produce values many orders of magnitude higher than that. Without any Lipschitz guarantees for the underlying function, such areas could create attractors for the ANNR, forcing the method's exploration to 'sink' in such singularities and over-exploit the area. 
In order to mitigate that, we propose to truncate the scores of simplices in accordance to a pre-selected sensible Lipschitz constant. 

Consider a simplex $\sigma \in \textnormal{Del}_P$ and its lifting $\hat{\sigma}$ and note that $\textnormal{Vol}(\hat{\sigma}) = \frac{1}{\cos{\alpha}}\textnormal{Vol}(\sigma)$, where $\alpha$ is the \textit{dihedral angle} between $\sigma$ and $\hat{\sigma}$ both naturally embedded in $\mathbb{R}^{m+1}$. Limiting the Lipschitz constant is equivalent to limiting the maximal dihedral angle to some $\alpha_0$. As the result of the clipping, $\textnormal{Vol}(\hat{\sigma})$ in Alg.~\ref{alg:nnr} gets transformed into $\min \{ \textnormal{Vol}(\hat{\sigma}), \frac{1}{\cos{\alpha_0}}\textnormal{Vol}(\sigma)\}$. In our experiments with gravitational waves, we use $\alpha_0=89^{\circ}$.

\textbf{Marginals.}
Fig.\ref{fig:all-marginals} presents the marginalizations of the approximated function over all possible two-parameter slices.

\begin{figure}[t!]
    \setlength{\tabcolsep}{0pt}
    \renewcommand{\arraystretch}{0}
    \centering
    \begin{tabular}{p{0.025\linewidth}@{}c@{}@{}c@{}@{}c@{}@{}c@{}@{}c@{}@{}c@{}@{}c@{}@{}c@{}@{}c@{}@{}c@{}@{}c@{}@{}c@{}@{}c@{}@{}c@{}@{}c@{}}
        & {\tiny $(d_L, \iota)$} 
        & {\tiny $(d_L, \psi)$} 
        & {\tiny $(d_L, \phi_c)$}
        & {\tiny $(d_L, a_1)$} 
        & {\tiny $(d_L, a_2)$}
        & {\tiny $(\iota, \psi)$}
        & {\tiny $(\iota, \phi_c)$}
        & {\tiny $(\iota, a_1)$}
        & {\tiny $(\iota, a_2)$}
        & {\tiny $(\psi, \phi_c)$}
        & {\tiny $(\psi, a_1)$}
        & {\tiny $(\psi, a_2)$}
        & {\tiny $(\phi_c, a_1)$}
        & {\tiny $(\phi_c, a_2)$}
        & {\tiny $(a_1, a_2)$}\\[0pt]
        
        \rotated{\hspace{2pt}\tiny Ground truth} & \includegraphics[align=c, width=.065\linewidth]{images/marg_true/0_1.png} & \includegraphics[align=c, width=.065\linewidth]{images/marg_true/0_2.png} & \includegraphics[align=c, width=.065\linewidth]{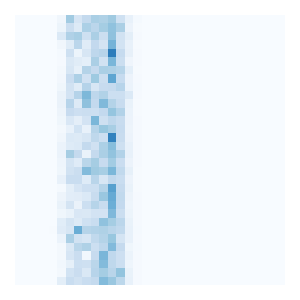} & \includegraphics[align=c, width=.065\linewidth]{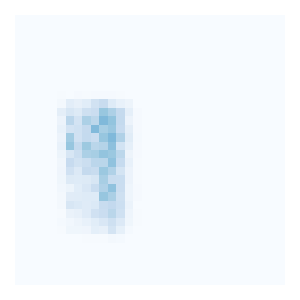} & \includegraphics[align=c, width=.065\linewidth]{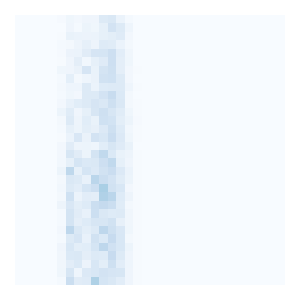} & \includegraphics[align=c, width=.065\linewidth]{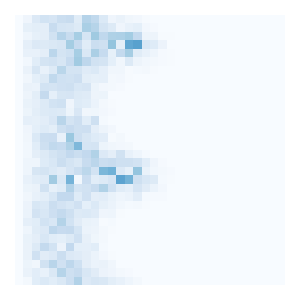} & \includegraphics[align=c, width=.065\linewidth]{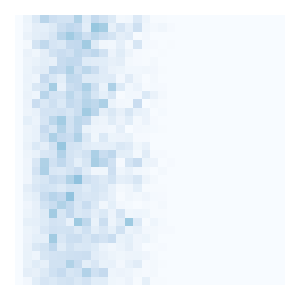} & \includegraphics[align=c, width=.065\linewidth]{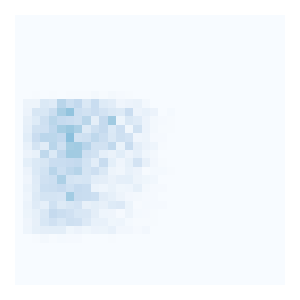} & \includegraphics[align=c, width=.065\linewidth]{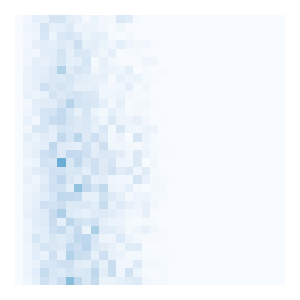} & \includegraphics[align=c, width=.065\linewidth]{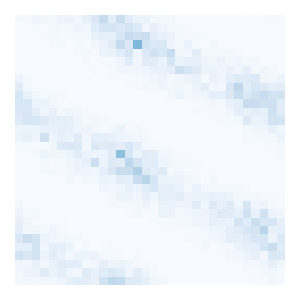} & \includegraphics[align=c, width=.065\linewidth]{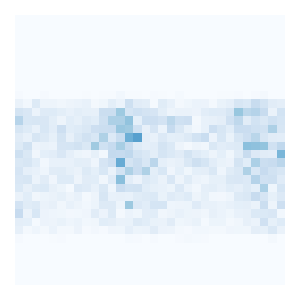} & \includegraphics[align=c, width=.065\linewidth]{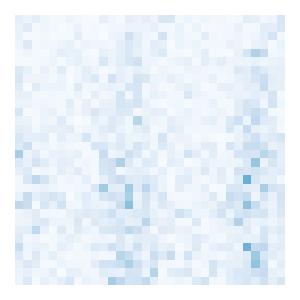} & \includegraphics[align=c, width=.065\linewidth]{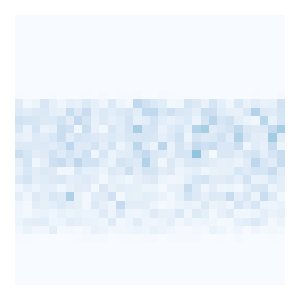} & \includegraphics[align=c, width=.065\linewidth]{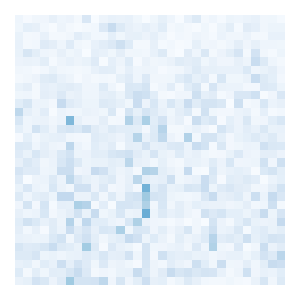} & \includegraphics[align=c, width=.065\linewidth]{images/marg_true/4_5.png}\\[0pt]
        
        \rotated{\hspace{2pt}\tiny DEFER} & \includegraphics[align=c, width=.065\linewidth]{images/marg_defer/0_1.png} & \includegraphics[align=c, width=.065\linewidth]{images/marg_defer/0_2.png} & \includegraphics[align=c, width=.065\linewidth]{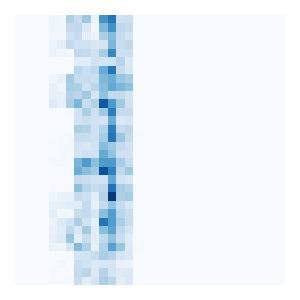} & \includegraphics[align=c, width=.065\linewidth]{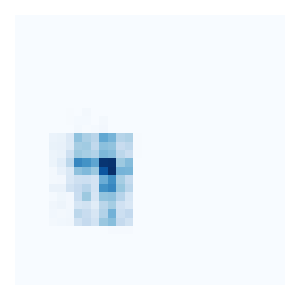} & \includegraphics[align=c, width=.065\linewidth]{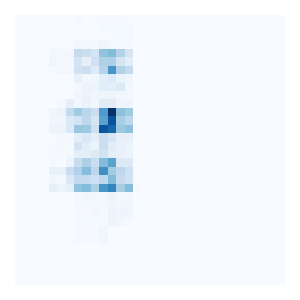} & \includegraphics[align=c, width=.065\linewidth]{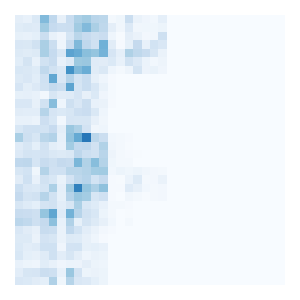} & \includegraphics[align=c, width=.065\linewidth]{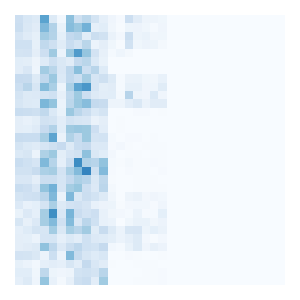} & \includegraphics[align=c, width=.065\linewidth]{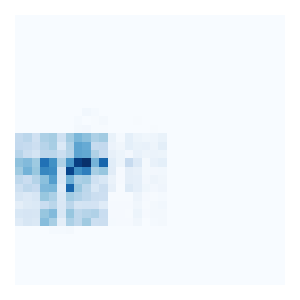} & \includegraphics[align=c, width=.065\linewidth]{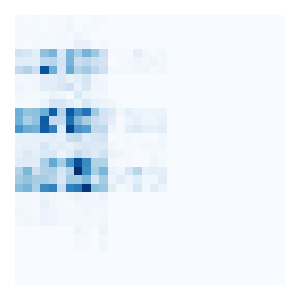} & \includegraphics[align=c, width=.065\linewidth]{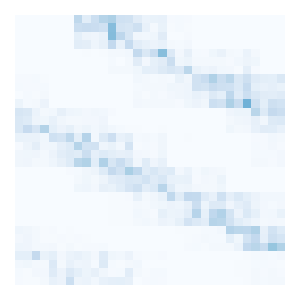} & \includegraphics[align=c, width=.065\linewidth]{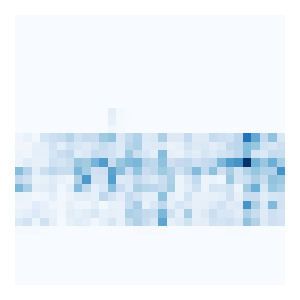} & \includegraphics[align=c, width=.065\linewidth]{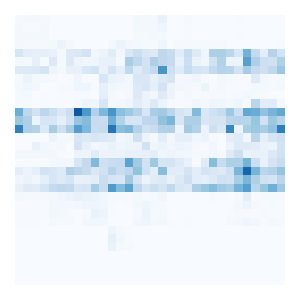} & \includegraphics[align=c, width=.065\linewidth]{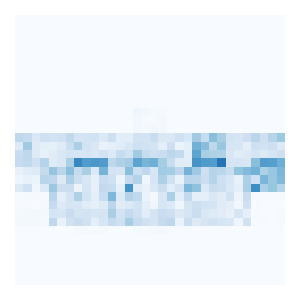} & \includegraphics[align=c, width=.065\linewidth]{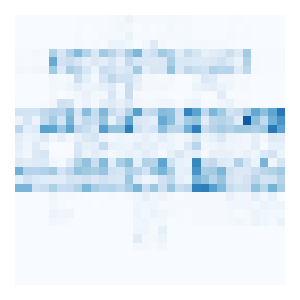} & \includegraphics[align=c, width=.065\linewidth]{images/marg_defer/4_5.png}\\[0pt]
        
        \rotated{\hspace{2pt}\tiny ANNR} & \includegraphics[align=c, width=.065\linewidth]{images/marg_annr/0_1.png} & \includegraphics[align=c, width=.065\linewidth]{images/marg_annr/0_2.png} & \includegraphics[align=c, width=.065\linewidth]{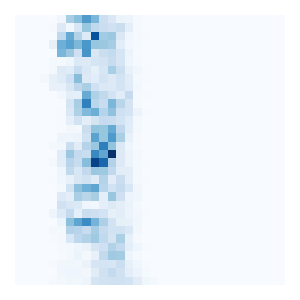} & \includegraphics[align=c, width=.065\linewidth]{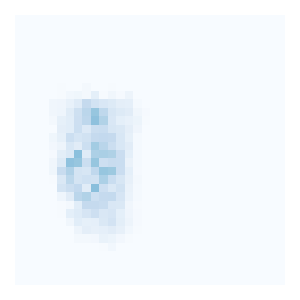} & \includegraphics[align=c, width=.065\linewidth]{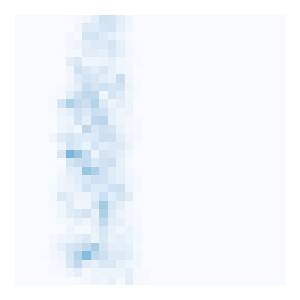} & \includegraphics[align=c, width=.065\linewidth]{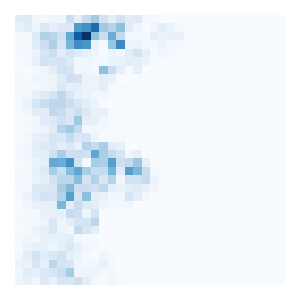} & \includegraphics[align=c, width=.065\linewidth]{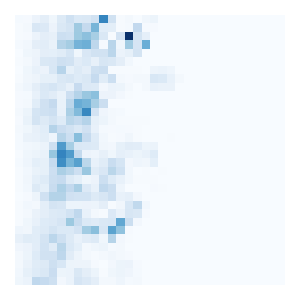} & \includegraphics[align=c, width=.065\linewidth]{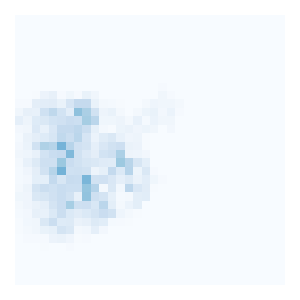} & \includegraphics[align=c, width=.065\linewidth]{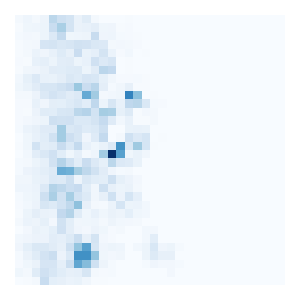} & \includegraphics[align=c, width=.065\linewidth]{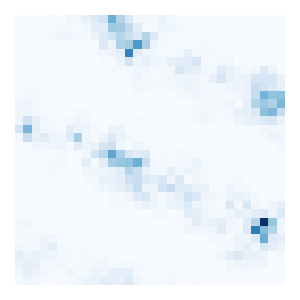} & \includegraphics[align=c, width=.065\linewidth]{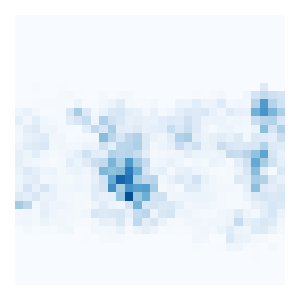} & \includegraphics[align=c, width=.065\linewidth]{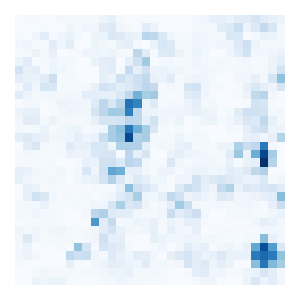} & \includegraphics[align=c, width=.065\linewidth]{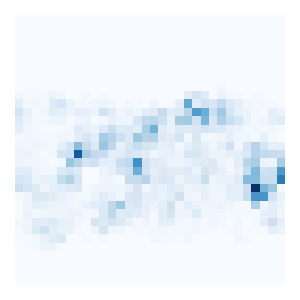} & \includegraphics[align=c, width=.065\linewidth]{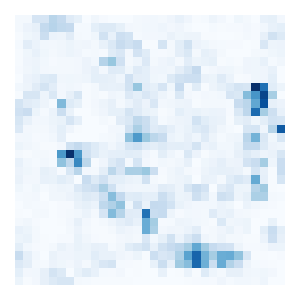} & \includegraphics[align=c, width=.065\linewidth]{images/marg_annr/4_5.png}\\[0pt]
        

    \end{tabular}
    \caption{Marginal distribution for gravitational waves over all two-dimensional slices of parameters. Each slice is identified by a pair of parameters named as in \cite{defer}.}
    \vspace{128in}
    \label{fig:all-marginals}
\end{figure}